\documentclass{article}
\usepackage[T1]{fontenc}
\usepackage{iclr2027_conference,times}
\usepackage{amsmath,amssymb,amsthm,mathtools}
\usepackage{booktabs,graphicx,microtype}
\usepackage[hidelinks]{hyperref}
\usepackage{url}

\newtheorem{theorem}{Theorem}
\newtheorem{lemma}{Lemma}
\newtheorem{proposition}{Proposition}

\newtheorem{definition}{Definition}
\theoremstyle{remark}

\usepackage{tcolorbox}
\tcbuselibrary{breakable}
\tcbset{
  thmbox/.style={breakable, colback=gray!6, colframe=black!45,
    boxrule=0.4pt, arc=1.5pt, boxsep=0pt, left=5pt, right=5pt, top=0.5pt, bottom=1pt,
    before skip=4pt plus 1pt minus 2pt, after skip=4pt plus 1pt minus 2pt, parbox=false},
  protocolbox/.style={thmbox, colback=teal!5, colframe=teal!60!black,
    title={Protocol: Certify-or-Measure}, fonttitle=\bfseries\small, coltitle=black,
    colbacktitle=teal!12, toptitle=0.5pt, bottomtitle=0.5pt, top=2pt, bottom=2pt},
  precommitbox/.style={colback=gray!4, colframe=black!55, boxrule=0.4pt, arc=1pt,
    boxsep=0pt, left=4pt, right=4pt, top=2pt, bottom=0pt, fonttitle=\footnotesize,
    coltitle=black, colbacktitle=gray!15, toptitle=1pt, bottomtitle=1pt,
    fontupper=\scriptsize, before skip=4pt, after skip=4pt}}
\tcolorboxenvironment{theorem}{thmbox}%

\newcommand{\KL}{D_{\mathrm{KL}}}
\newcommand{\dbin}{d_{\mathrm{bin}}}
\newcommand{\cenc}{c_{\mathrm{enc}}}
\newcommand{\cdec}{c_{\mathrm{dec}}}
\newcommand{\qbar}{\bar q}
\newcommand{\Aq}{\mathcal A_{q}}
\newcommand{\Ap}{\mathcal A_{\mathrm{prior}}}
\newcommand{\Amu}{\mathcal A_{\mathrm{mean}}}
\newcommand{\Agr}{\mathrm{Agr}}
\newcommand{\cZ}{\mathcal Z}
\newcommand{\cX}{\mathcal X}
\newcommand{\E}{\mathbb E}
\newcommand{\R}{\mathbb R}
\newcommand{\Rbar}{\bar R}
\newcommand{\gap}{\bar\Delta}

\title{Perfect Reconstruction Cannot Show\\
Whether Two Codebooks\\ on a VAE's Latent Space Agree}

\author{Yusuke Hayashi\thanks{Also affiliated with Artificial Life Institute and Humanity Brain.}\\
AI Alignment Network}

\iclrfinalcopy
\hypersetup{pdftitle={Perfect Reconstruction Cannot Show Whether Two Codebooks on a VAE's Latent Space Agree},pdfauthor={Yusuke Hayashi}}
\begin{document}
\maketitle
\lhead{Preprint}

\begin{abstract}
Many systems read the latent space of a variational autoencoder (VAE) as
discrete codes, and rely on each latent being read back as the code it was
written as.  \citet{shannon1948} took such a shared codebook as given; inside
a VAE nothing enforces it.  We prove that no item of a standard training
report can reveal whether the written and read codes agree on encoded data.
We construct two
VAEs that reconstruct every input perfectly, infer exactly, and match on every
per-example ELBO and every code histogram, yet agreement is one in the first
and zero in the second.  We then show that the ELBO's own terms bound how far
a measured agreement can move: the variational gap bounds the shift to exact
inference, and the rate bounds the shift to generation from the prior.
For Gaussian posteriors, a mean-field encoder pays an exact price in the gap
that depends on its choice of axes.  On probabilistic PCA the gap certifies a
trained encoder in all 75 settings we test, at agreement between 0.76 and 0.97.
On 240 trained VAEs, a floor set by the rate and the dataset size rules out rate-side
certification in advance for 66 models, and none of them certifies.  The
results yield a simple protocol, Certify-or-Measure, for any pipeline that
reads a VAE as codes.
\end{abstract}

\section{Introduction}
\label{sec:intro}

\begin{figure}[!h]
\centering
\includegraphics[width=0.74\linewidth]{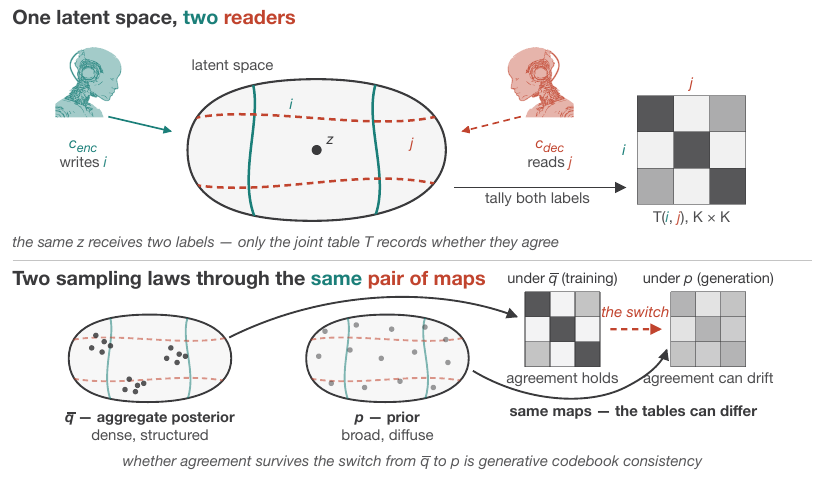}
\caption{\textbf{A pipeline that reads a VAE's latent space as codes holds two
codebooks on one space, and their agreement can change when the model
generates.}  Top: $\cenc$ writes a latent $z$ as code $i$, $\cdec$ reads the
same $z$ as code $j$, and the joint table $T(i,j)$ counts the pairs; agreement
is its diagonal mass.  Bottom: the same maps under the aggregate posterior
$\qbar$ of encoded data (agreement $\Aq$) and under the prior $p$ used for
generation ($\Ap$).  No item of a standard training report reveals $\Aq$
(Theorem~\ref{thm:invisibility}); the rate bounds the drift from $\Aq$ to $\Ap$,
the failure of generative codebook consistency
(Theorem~\ref{thm:certification}b).}
\label{fig:concept}
\end{figure}

\begin{figure}[!b]
\centering
\includegraphics[width=0.86\linewidth]{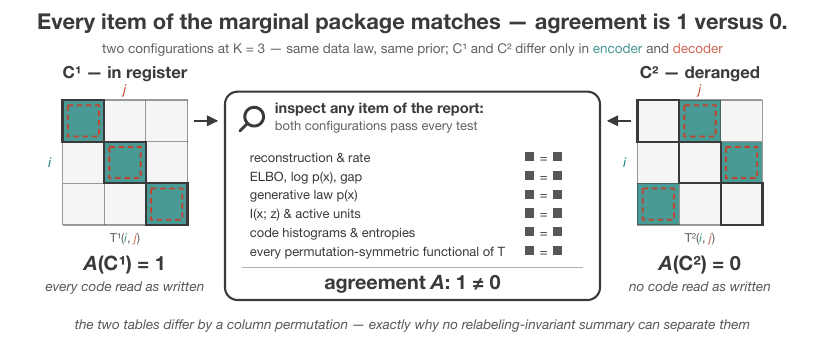}\\[2pt]
\includegraphics[width=\linewidth]{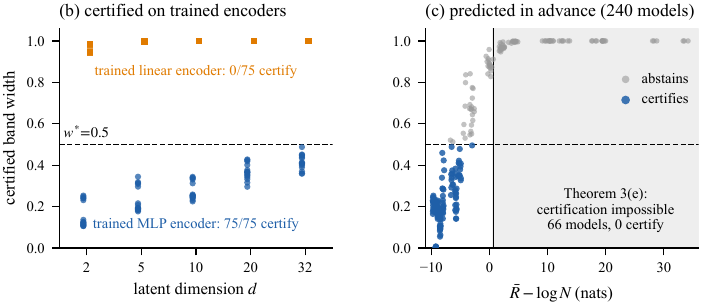}
\caption{\textbf{No item of a standard training report tells the two
configurations in (a) apart, yet the ELBO's own terms certify agreement on
trained encoders (b) and predict in advance where no certificate can exist
(c).}  (a)~Top: the witness pair of Theorem~\ref{thm:invisibility} at $K=3$.
Rows give the code a latent is written as ($\cenc$), columns the code it is
read as ($\cdec$), and agreement is the mass on the outlined diagonal.  Both
configurations share the data, the prior and both code maps, reconstruct every
input perfectly with zero variational gap, and match on every item of the marginal package (Definition~\ref{def:package}).  (b)~Probabilistic PCA, where $\log p(x)$ is
exact: width of the band that the variational gap places around each measured
agreement (Theorem~\ref{thm:certification}a), $15$ cells per latent dimension
($5$ folds, $3$ code maps).  A band certifies when its width is at most
$w^*=1/2$.  The trained mean-field MLP encoder certifies in every cell, and
every band contains the agreement under the exact posterior
(Appendix~\ref{app:nativerank}).  (c)~Rate-side band width for $240$ trained VAEs against $\Rbar-\log N$, the average
rate minus the log size of the audited split.  Theorem~\ref{thm:price}(e)
forbids certification in the shaded region, $\Rbar-\log N>\log2$ (Section~\ref{sec:exp-scale}).}
\label{fig:hero}
\end{figure}

Many systems read the latent space of a learned model as a small set of
discrete codes.  They cluster latents \citep{banerjee2005clustering}, quantize
them into tokens \citep{vandenoord2017neural} or for compression
\citep{farvardin1990vq}, and read latent units as interpretable factors
\citep{higgins2017beta,chen2018isolating}.  Every such pipeline relies on a
contract it rarely states: a latent written as code~$i$ is read back as
code~$i$.  \citet{shannon1948} built communication theory on a codebook that
transmitter and receiver share by assumption, and \citet{weaver1949} set the
meaning of the symbols aside as a separate level.  Inside one learned model,
nothing guarantees that the codebook is shared.  This paper asks whether it
is, and whether anyone would notice if it were not.

A variational autoencoder (VAE) \citep{kingma2014auto,rezende2014stochastic}
is the natural place to ask.  A pipeline that reads its latent space through
$K$ codes declares two \emph{code maps} on that space before it writes
anything.  The encoder-side map $\cenc$ gives the code a latent is written as;
in practice it is fitted on encoder statistics.  The decoder-side map $\cdec$
gives the code the latent is read as; in practice it is a nearest-prototype
rule on decoded outputs.  Both are fixed partitions of the latent space, so
the pipeline holds two codebooks on one space (Figure~\ref{fig:concept}, top).
Their \emph{agreement} is the probability that one latent receives the same
code from both.  We write $\Aq$ for agreement on encoded data and $\Ap$ for
agreement when the model generates from its prior (Figure~\ref{fig:concept},
bottom).  Training never asks for
agreement.  Practitioners watch a standard report instead: the evidence lower
bound (ELBO) of each example with its reconstruction term and its \emph{rate} (its KL divergence to the prior;
\citealp{alemi2018fixing}), the log-likelihood, the information the latent
carries about the input, active units, and how often each code is used.
Reconstruction ties the decoder to the encoder, so a clean report looks like
strong evidence that the two ends read one codebook.

In this work, we show that the report cannot settle the question, and that
two quantities the ELBO already provides bound how far a measured agreement
can move.  Theorem~\ref{thm:invisibility}
gives the negative result (Figure~\ref{fig:hero}a).  For every codebook size
we construct two configurations that share the data, the prior and both code
maps.  Both reconstruct every input perfectly, and in both the
\emph{variational gap}, the amount by which the ELBO falls below
$\log p(x)$, is zero.  They match on every item of the standard report
(Definition~\ref{def:package}), example by example and code by code, and they
generate the same data distribution, so no sample-quality score such as FID
\citep{heusel2017gans} can separate them.  Yet $\Aq$ is one in the first
configuration and zero in the second.  The joint table of written against
read codes, kept in the pipeline's own indices, shows the difference at once.

The positive result is that the variational gap and the rate, read as
\emph{budgets} in nats, bound how far a measured agreement can move
(Theorem~\ref{thm:certification}).  The gap bounds the distance between $\Aq$
and its value under exact inference; the rate bounds the distance between
$\Aq$ and $\Ap$.  Both bounds are stated in the divergence between two coin
flips, so they turn a budget into an interval, and a budget $B$ allows no
interval narrower than $1-e^{-B}$ (Lemma~\ref{lem:width}).  A budget must be a
proved upper bound to be \emph{licensed}; common plug-in estimates are lower
bounds and are not (Lemma~\ref{lem:budgets}).  For Gaussian posteriors,
Theorem~\ref{thm:price} gives the exact price that a mean-field encoder, which
works coordinate by coordinate, pays in the gap: it depends on the choice of axes and vanishes in
the posterior's own axes, while the certified quantities do not depend on the
axes at all.  The same theorem gives a floor: a latent cannot carry more than
$\log N$ nats about the $N$ examples over which the rate is averaged, so a
rate above $\log N+\log 2$ rules out a rate-side band narrower than one half
before any audit is run.

Trained models behave as the theory predicts (Section~\ref{sec:experiments}).
On probabilistic PCA, where the likelihood is exact, the gap certifies a
trained mean-field encoder in all $75$ cells of a sweep over latent
dimension, at agreement between $0.76$ and $0.97$, and every certified band
contains the exact value (Figure~\ref{fig:hero}b).  A VAE trained on MNIST
loses agreement when it generates, $\Aq=0.770$ on encoded data against
$\Ap=0.525$ under the prior over five seeds, and the same drift appears
in probabilistic PCA where every quantity is exact.  On the same probabilistic PCA models, a rotation of the latent axes moves the mean-field budget from $0.000$ to $0.474$ nats while the
certified quantities stay fixed, as Theorem~\ref{thm:price} predicts.  On six SVHN models, a new licensed budget built from the rate is $2.5$ to
$3$ times tighter than the
best previously available one \citep{kolchinsky2017estimating}.  Across $240$
trained VAEs on four datasets, the dataset floor rules out rate-side certification in advance for $66$
models, and none of them certifies (Figure~\ref{fig:hero}c).

\paragraph{Contributions.}
\begin{itemize}
\item We prove that no item of a standard training report, including the
per-example ELBO at perfect reconstruction and exact inference, can reveal
whether a VAE's two code maps agree on encoded data
(Theorem~\ref{thm:invisibility}, Definition~\ref{def:package},
Figure~\ref{fig:hero}a).
\item We show that the variational gap and the rate bound how far measured
agreement can move under exact inference and under generation, prove that a budget $B$ allows no band narrower than $1-e^{-B}$, and state which budget estimates are
licensed (Theorem~\ref{thm:certification}, Lemmas~\ref{lem:width}
and~\ref{lem:budgets}).
\item We price the certificate exactly for Gaussian posteriors under
mean-field computation, and show that the rate and the dataset size decide
in advance when rate-side certification is impossible
(Theorem~\ref{thm:price}).
\item We assemble the results into a protocol, Certify-or-Measure, and test both certificates on trained models: each fires only where the
theory says it can, and none fires where the dataset floor rules it out (Sections~\ref{sec:protocol} and~\ref{sec:experiments},
Figure~\ref{fig:hero}b,c).
\end{itemize}

A clean training report does not show that the code a pipeline writes is the
code it reads.  Log the joint code table, and certify from the ELBO's own
terms when their budgets are small enough.

\section{Setup: two code maps, one channel}
\label{sec:setup}

Data $x\sim\mathcal D$; encoder $q(z\mid x)$ on $\cZ$ with prior $p$;
decoder $p(x\mid z)$; aggregate posterior $\qbar=\E_xq(\cdot\mid x)$.  A
\emph{configuration} is
$C=(\mathcal D,\,q,\,p(x\mid z),\,\cenc,\cdec)$ where
$\cenc,\cdec:\cZ\to[K]$ are fixed measurable code maps.  The maps are
taken as given.  In applications $\cenc$ is fit on encoder statistics and
$\cdec$ is a nearest-prototype rule in the decoder's Bregman output
geometry \citep{banerjee2005clustering}.  Theorem~\ref{thm:certification}
holds for any fixed pair, Theorem~\ref{thm:invisibility} constructs its own,
and the experiments declare theirs.  Each code map partitions $\cZ$ into
$K$ labeled cells, a codebook on the latent space, so the pair gives two
codebooks on one space.

For a sampling law $\mu$ on $\cZ$ the pair $(\cenc,\cdec)$ has joint table
$T^{\mu}(i,j)=\mu(\{\cenc=i\}\cap\{\cdec=j\})$ (Figure~\ref{fig:concept}).  Agreement under $\mu$ is
$\mathcal A[\mu]=\mu(\Agr)$ with $\Agr=\{z:\cenc(z)=\cdec(z)\}$; we write
$\Aq=\mathcal A[\qbar]$ (posterior sampling) and $\Ap=\mathcal A[p]$ (prior
sampling), and $\dbin(a\Vert b)=a\log\frac ab+(1-a)\log\frac{1-a}{1-b}$.
A third reading appears in the experiments: $\Amu$, agreement of the
\emph{mean} reading (both maps evaluated at the posterior means
$\E[z\mid x]$ rather than at draws).

\emph{Notation.}  $[K]=\{1,\dots,K\}$; $S$ is the number of posterior
draws per datum when sampling is Monte Carlo.  Two ELBO-derived scalars
recur: the average rate $\Rbar=\E_x\KL(q(\cdot\mid x)\Vert p)$ and the
average variational gap $\gap=\E_x[\log p(x)-\mathrm{ELBO}(x)]$; $\Rbar$ is logged in every VAE run and $\gap$ wherever $\log p(x)$ is
tractable, and
Section~\ref{sec:certification} reads them as \emph{budgets}; an estimator
is \emph{licensed} when it is a certified upper bound on the quantity it
replaces (Section~\ref{sec:budgets}), and the certified quantity is the
\emph{estimand}.  For a configuration we write
$\mathcal A(C)=\mathcal A[\qbar_C]$, and we apply $\mathcal A$ to a joint
table where convenient, $\mathcal A(T)=\sum_iT(i,i)$, which is the same
functional in table coordinates.

\begin{definition}[Marginal package]\label{def:package}
The \emph{marginal package} $M(C)$ of a configuration comprises:
(m1) per-example reconstruction $\E_{q(z\mid x)}\log p(x\mid z)$;
(m2) per-example rate $\KL(q(\cdot\mid x)\Vert p)$;
(m3) per-example ELBO, marginal likelihood $\log p(x)$, variational gap
$\Delta(x)=\log p(x)-\mathrm{ELBO}(x)$, and the generative law $p(x)$ as a
distribution;
(m4) $I(x;z)$ under the encoder joint and the per-dimension active-units
statistics $\mathrm{AU}_k=\mathrm{Var}_x\,\E[z_k\mid x]$;
(m5) encoder and decoder code histograms under $\qbar$, their entropies and
occupied-code counts;
(m6) every permutation-symmetric functional of the joint table $T^{\qbar}$
(every $F$ with $F(P_1TP_2^{\top})=F(T)$ for permutation matrices $P_1,P_2$),
including $I(\cenc;\cdec)$, all R\'enyi mutual informations, joint entropy,
both histograms as multisets, and the permutation-matched agreement
$\max_{\pi}\sum_iT(i,\pi(i))$;
(m7) prior-law agreement $\Ap$.
\end{definition}

Items (m1)--(m5) are what standard practice logs, and (m6) says that reporting
the joint table up to relabeling does not help either.  The name
\emph{marginal} is short and understates the package: reconstruction, the ELBO
and $\log p(x)$ couple encoder to decoder, (m6) is a functional of the coupled
table, and (m7) is agreement under the prior.

\paragraph{Scope.}
Invisibility concerns agreement in the pipeline's own indices under fixed
maps, the quantity that matters when a pipeline writes code $i$ and expects
code $i$ back.  In the pair of Theorem~\ref{thm:invisibility}(I) the code
histograms are identical index by index, and the two joint tables are column
permutations of each other, so reporting $T$ both detects and repairs this
pair.  Proposition~\ref{prop:a4} gives a closed-form pair with equal
histograms and equal $I(\cenc;\cdec)$ whose tables are not relabelings of each
other; Proposition~\ref{prop:structure} shows that matched summaries also
leave the error structure free.

\section{Invisibility}
\label{sec:invisibility}

\begin{theorem}[Invisibility]\label{thm:invisibility}
\textbf{(I) One pair, whole package, shared data.}
For every $K\ge2$ there exist configurations $C^1,C^2$ with the same data
law and prior, differing only in encoder and decoder, with
$M(C^1)=M(C^2)$---every item of Definition~\ref{def:package}, per example
and in aggregate, to exact equality, including $p(x)$ as a
distribution---while $\mathcal A(C^1)=1$ and $\mathcal A(C^2)=0$.

\textbf{(I$'$) Moment blindness with a sharp threshold.}
For every $K\ge2$ and $m\ge2$ the pair in (I) can be built so that, in
addition: (N1) per-example posterior means are identical example by example,
so every functional of the means---per-dimension AU, their full covariance,
any statistic---coincides; (N2) every aggregate-posterior mixed moment of
degree below $m$ in the coordinates that $\cdec$ reads coincides; (N3) the
threshold is sharp, with explicit degree-$m$ cross moments separating the
pair.  The blindness order $m$ is a design parameter independent of $K$.

\textbf{(II) Orbit form, for trained models.}
Every prior-preserving bimeasurable bijection $\tau$ of $\cZ$ leaves
(m1)--(m3) and $I(x;z)$ exactly invariant for every $x$ (and per-dimension
AU when $\tau^{-1}$ acts as a translation on $\mathrm{supp}(\qbar)$), while
agreement transforms as $\mathcal A(C_\tau)=\qbar(\tau(\Agr))$; the full
statement and proof are in Appendix~\ref{app:orbit}.

\end{theorem}

\paragraph{The grid witness.}
Partition $[0,1)^2$ with a uniform prior into $K^2$ cells, with
$\cenc=\text{row}$ and $\cdec=\text{column}$.  $C^1$ couples the data atoms
to the diagonal cells and $C^2$ to a derangement $\sigma$ of them
(Figure~\ref{fig:hero}a).  Inference is exact in both configurations and
every package item coincides.  This includes $p(x)$, which is uniform over
the same $K^2$ atoms, so no statistic of generated data, from two-sample tests
to sample-quality scores such as FID and KID
\citep{heusel2017gans,binkowski2018demystifying}, can separate the pair.  Yet
agreement is $1$ versus $0$.  Appendix~\ref{app:grid} gives the computation
and shows that Gaussian emissions preserve the equality exactly.

\paragraph{Companions.}
Proposition~\ref{prop:a4} (Appendix~\ref{app:a4}) carries the impossibility
to non-uniform marginals: encoder and decoder histograms are equal to the same
$\nu$ index by index, $I(\cenc;\cdec)$ is matched exactly, and agreement differs ($0$ versus positive when $\max_i\nu_i\le\frac12$),
with exact feasible set $[\max(0,2\max_i\nu_i-1),1]$.
The grid pair still differs in one statistic outside the package, the
per-example posterior means; the interleaved witness of Part~(I$'$) removes
that difference (Lemma~\ref{lem:harmonic}, Appendix~\ref{app:interleaved}).
Appendix~\ref{app:proofs} collects all proofs, and Appendix~\ref{app:scope}
states what the theorem does not claim and records the numerical verification
of both witness constructions to machine precision.

\section{Two-sided certification}
\label{sec:certification}

The standard report cannot detect the failure, even at perfect
reconstruction.  Two scalars derived from the ELBO, the gap and the rate, can
bound it, one each.
Recall $\gap\ge0$ and $\Rbar$ from Section~\ref{sec:setup}; write
$p_\theta(\cdot\mid x)$ for the \emph{exact} posterior of the model,
$\bar p_\theta=\E_x\,p_\theta(\cdot\mid x)$ for its aggregate, and
$\bar\eta_p=\E_x\,p_\theta(\{\cenc\ne\cdec\}\mid x)$ for the
disagreement rate under exact-posterior sampling.  Throughout this section
assume $q(\cdot\mid x)\ll p_\theta(\cdot\mid x)$ for $\mathcal D$-a.e.~$x$.

\begin{theorem}[Two-sided certification]
\textbf{(a) Reconstruction side.}\label{thm:certification}
For any fixed measurable $E\subseteq\cZ$,
$\dbin\bigl(\qbar(E)\,\Vert\,\bar p_\theta(E)\bigr)\le\gap$;
with $E=\{\cenc\ne\cdec\}$,
\[
\dbin\bigl(1-\Aq\,\Vert\,\bar\eta_p\bigr)\;\le\;\gap ,
\]
and in fact $\gap$ decomposes as the sum of this binary-KL term, a Jensen
term, and a within-cell term, all non-negative.

\textbf{(b) Generative side.}
For the same event,
\[
\dbin\bigl(\Aq\,\Vert\,\Ap\bigr)\;\le\;\KL(\qbar\Vert p)\;\le\;\Rbar .
\]
\end{theorem}

\noindent\emph{Proof.}  Appendix~\ref{app:thm2}; (a) is the binary
specialization of the post-processing decomposition of the gap, (b) is
data processing plus the rate decomposition \citep{hoffman2016elbo}.

Part (a) certifies the posterior-sampling audit and part (b) the generative
one.  Inverting (b) at an observed $\Aq$ and a budget
$B\in\{\KL(\qbar\Vert p),\Rbar\}$ yields a certified band
$\{a:\dbin(\Aq\Vert a)\le B\}\ni\Ap$.  Its one-dimensional form loses nothing,
because by Theorem~\ref{thm:invisibility}(I) richer permutation-symmetric
summaries cannot flag the failure anyway.  All of its leverage comes from the
budget, whose price and floor are the subject of Section~\ref{sec:price}.

\section{The price of certification, and Certify-or-Measure}
\label{sec:price}

Theorem~\ref{thm:certification} is stated in exact quantities, but mean-field
training computes rates coordinate-wise and certificates inherit the coordinate
system.  This section prices that inheritance exactly for Gaussian
posteriors: exactly in linear-Gaussian models, and per example for the
Gaussian approximation otherwise.

Let $\cZ=\R^d$ and the posterior at $x$ be $N(m_x,\Sigma)$ with precision
$\Lambda=\Sigma^{-1}\succ0$ and eigenvalues $\lambda$, condition number
$\kappa$.  For a basis $Q\in O(d)$ let
$\mathcal F_Q=\{N(\mu,Q^{\top}DQ):D\ \text{diagonal}\succ0\}$ be the
mean-field family in basis $Q$, and
$B_Q=\min_{q\in\mathcal F_Q}\KL(q\Vert N(m_x,\Sigma))$ the smallest
certificate budget achievable in that basis.

\begin{theorem}[The coordinate price of certification]\label{thm:price}
Write $\Lambda_Q=Q\Lambda Q^{\top}$.

\textbf{(a) Exact price and worst basis.}
$B_Q=\tfrac12\log\bigl[\det\mathrm{diag}(\Lambda_Q)/\det\Lambda_Q\bigr]$,
attained at $\mu^*=m_x$, $D^*_{ii}=1/(\Lambda_Q)_{ii}$.  Moreover
(i) $B_Q=0$ iff $\Lambda_Q$ is diagonal---exactly on the posterior
eigenbasis, up to axis permutations, signs, and within-eigenspace
rotations;
(ii) $\sup_Q B_Q=\tfrac d2\log\bigl(\mathrm{AM}(\lambda)/
\mathrm{GM}(\lambda)\bigr)$, finite and attained (Schur--Horn,
\citealp{marshall2011inequalities});
(iii) restricting the product family to Gaussians is without loss.

\textbf{(b) Estimand invariance under transport.}
For $R\in O(d)$, transport the configuration: $z\mapsto Rz$, decoder
$g'(w)=g(R^{\top}w)$, posterior $N(Rm_x,R\Sigma R^{\top})$, prior
$R_{\#}p$, and \emph{transported} maps $c'(w)=c(R^{\top}w)$.  Then the
joint law of $(c'_{\mathrm{enc}},c'_{\mathrm{dec}})$ under every
transported sampling law equals the original joint law; in particular
$\Aq$ under full-covariance posteriors, $\Ap$, and every functional of
code-pair laws are exactly invariant.

\textbf{(c) The asymmetry.}
For a fixed model, $Q\mapsto B_Q$ ranges over the closed interval
$\bigl[0,\tfrac d2\log(\mathrm{AM}/\mathrm{GM})\bigr]$ while every
transported estimand is constant; across models the price is unbounded
($\kappa\to\infty$).  A mean-field-budget certificate is a property of the
(model, basis) pair; the certified quantity is a property of the model.
Reporting a mean-field KL budget without its basis is ill-posed; the
price-minimizing basis is the posterior eigenbasis.  The declared code map is the
exact dual.  Theorem~\ref{thm:certification} holds for \emph{every} fixed
measurable pair $(\cenc,\cdec)$ at once, and a rate budget is a functional of
$\qbar$ alone, so changing the map moves the certified \emph{quantity} and leaves
the \emph{budget} fixed---the mirror of the basis, which moves the budget and
leaves the quantity fixed.  Agreement is a property of the (model, declared map)
pair and the budget a property of the model; reporting either without the other
is ill-posed, and Section~\ref{sec:exp-price} measures both halves on one
checkpoint.

\textbf{(d) Informative floor.}
Let $y\in[L]$ be any finite-alphabet label with $y\perp z\mid x$, and let
$\alpha$ be the held-out accuracy of any predictor of $y$ from the
\emph{sampled} code $\cenc(z)$, $z\sim q(\cdot\mid x)$.  Then with
$h_2(\cdot)$ the binary entropy in nats and
$\gamma(\alpha)=H(y)-h_2(1-\alpha)-(1-\alpha)\log(L-1)$ (Fano's
inequality, \citealp{cover2006elements}),
\[
\Rbar\;=\;\KL(\qbar\Vert p)+I(x;z)\;\ge\;I(x;z)\;\ge\;I(x;\cenc(z))
\;\ge\;I(y;\cenc(z))\;\ge\;\gamma(\alpha),
\]
and by Lemma~\ref{lem:width} every band
$\{a:\dbin(\mathcal A\Vert a)\le B\}$ with $B\ge\gamma$ has width at least
$1-e^{-\gamma}$.  At $\gamma\ge1$ nat the width is $\ge1-e^{-1}\approx0.632$
regardless of the observed agreement.

\textbf{(e) Dataset floor.}  If the data law is empirical over $N$ points---the
same $N$ points that define $\qbar=\tfrac1N\sum_iq(\cdot\mid x_i)$ and
$\Rbar$---then $I(x;z)\le H(x)=\log N$, so
$\KL(\qbar\Vert p)\ge\Rbar-\log N$: \emph{every} licensed budget obeys
$B\ge\Rbar-\log N$, the band has width at least $1-e^{-(\Rbar-\log N)}$,
and certification at width $w^*$ is \textbf{impossible} whenever
$\Rbar>\log N-\log(1-w^*)$---no estimator, however tight, can change this.
\end{theorem}

The optimal mean-field Gaussian fit in (a) is classical
\citep{bishop2006pattern}; what is new is reading its value as a
\emph{price paid in $Q$} and pairing it with the invariance statement (b).

\begin{lemma}[Width]\label{lem:width}
For $a\in[0,1)$ and $B>0$ let $W(a,B)$ be the width of
$\{x:\dbin(a\Vert x)\le B\}$.  Then $W(a,B)\ge1-e^{-B}$, with equality iff
$a=0$.
\end{lemma}

Inverting it (proof in Appendix~\ref{app:thm3}) gives a \emph{necessary} admission condition: width
$\le w^*$ requires $B\le-\log(1-w^*)$ nats ($0.693$ at $w^*=0.5$).  It is
not sufficient, because the width also depends on the observed agreement.
Writing $B^*(a)$, the \emph{admission budget}, for the largest budget
whose band at agreement $a$ has width $\le w^*$: at $w^*=0.5$,
\[
B^*(0)=0.693,\quad B^*(0.1)=0.370,\quad B^*(0.25)=0.195,\quad B^*(0.5)=0.144,
\]
symmetrically above $\tfrac12$.  The requirement is thus $1.9$--$4.8$ times
\emph{stricter} than the $0.693$ headline across this range, and it is
$B\le B^*(\Aq)$, not the headline, that the protocol of
Section~\ref{sec:protocol} admits on.

Appendix~\ref{app:thm3} proves (a)--(c), with the exact zero-set coset, a
Kantorovich corollary, and the fact that $B_Q$ depends only on
$\mathrm{Corr}(\Lambda_Q)$; the supremum comes from the
Schur--Horn characterization
$\{\mathrm{diag}(Q\Lambda Q^{\top}):Q\in O(d)\}=\{u:u\prec\lambda\}$.
Three scope notes on (b), on refitting versus transporting maps, the
exclusion of mean-field $\Aq$, and padding, are collected in
Appendix~\ref{app:price-scope}.

\subsection{Which budgets are licensed}
\label{sec:budgets}

Theorem~\ref{thm:certification}(b) admits any $B\ge\KL(\qbar\Vert p)$, so the
protocol's soundness rests entirely on the estimator being an \emph{upper}
bound.  The two most natural tightenings are both wrong in the direction that
matters.

\begin{lemma}[Budget validity]\label{lem:budgets}
Let $f$ be any measurable coarsening of $\cZ$ and $\tilde q$ the
moment-matched Gaussian of $\qbar$.  Then
\[
\KL(f_{\#}\qbar\Vert f_{\#}p)\;\le\;\KL(\qbar\Vert p),\qquad
\KL(\tilde q\Vert p)\;\le\;\KL(\qbar\Vert p)\ \text{when }p\text{ is Gaussian},
\]
so binned/grid plug-ins (data processing) and moment-matched Gaussian plug-ins
(maximum entropy at fixed moments) are \emph{lower} bounds and are not licensed
budgets.  By contrast $\Rbar$ is licensed but loose by exactly $I(x;z)$; and
for any certified lower bound $\hat I\le I(x;z)$, $\;\Rbar-\hat I$ is licensed.
\end{lemma}

The last clause is the constructive one
(proof in Appendix~\ref{app:budgets-proof}).  For a VAE, $\qbar$ is a finite
Gaussian mixture with an exactly computable density, so
$I(x;z)=\E_{i,\,z\sim q_i}[\log q_i(z)-\log\qbar(z)]$ is estimated
\emph{without bias}, unlike contrastive bounds (capped at $\log M$;
\citealp{oord2018cpc,poole2019variational}) or
small-mixture plug-ins (floored at $\Rbar-\log N'$).  The statistic is
clipped below at $-c$, which can only \emph{raise} its mean, and by at most
$e^{-c}$ (Lemma~\ref{lem:clip}); subtracting that amount from an
empirical-Bernstein \citep{maurer2009empirical,audibert2009exploration}
lower confidence bound on the clipped mean gives a certified $\hat I$, and
the resulting mutual-information budget $B_{\mathrm{MI}}=\Rbar-\hat I$ is
licensed at level $1-\delta_B$ (Appendix~\ref{app:mi-lemmas}).  This is
what moves the protocol off its abstention.

\subsection{The Certify-or-Measure protocol}
\label{sec:protocol}

\begin{tcolorbox}[protocolbox]
\noindent Inputs: an estimate
$\hat{\mathcal A}_q$ of $\Aq$ with a $(1-\delta_q)$ confidence interval
$C_q$; a budget $B$ licensed at level $1-\delta_B$ ($\delta_B=0$ for a
logged $\Rbar$); a pre-declared width threshold $w^*\in(0,1)$, failure
probability $\delta$, and draw count $n$.
\textbf{(1)} Form the band $U=\bigcup_{a\in C_q}\{b:\dbin(a\Vert b)\le B\}$,
an interval (Theorem~\ref{thm:certification}b, Lemma~\ref{lem:aq}).
\textbf{(2)} If its width is $\le w^*$
(for an exact $\Aq$, equivalently $B\le B^*(\Aq)$, the admission budget
of Section~\ref{sec:price}), \textbf{certify}: report $U$, which contains
$\Ap$ with probability $\ge1-\delta_B-\delta_q$.
\textbf{(3)} Else \textbf{abstain and measure}: draw $z_1,\dots,z_n\sim p$,
read both codes, report $\hat{\mathcal A}\pm\sqrt{\ln(2/\delta)/(2n)}$
\citep{dvoretzky1956asymptotic}, at one decoder pass per draw.
\end{tcolorbox}

The protocol is sound in both branches: the certify branch by
Theorem~\ref{thm:certification}(b) and a union bound over the two
Monte-Carlo failure events (budget and agreement estimate, both on the
audited split; Appendix~\ref{app:mi-lemmas}), the measure
branch by DKW.  Its abstention region is a theorem, not an inconvenience.  Under the rate
budget, Theorem~\ref{thm:price}(d) \emph{forces} abstention for
informative codebooks.  The tighter budget $\KL(\qbar\Vert p)$ carries no
such floor, since a model whose aggregate posterior matches its prior
certifies legitimately at any informativeness, so the protocol prefers it
when it can be estimated.  Coordinate inflation of a
mean-field budget (Theorem~\ref{thm:price}a,c) and finite-sample
resolution floors only widen bands, so certify decisions remain sound.

\section{Experiments}
\label{sec:experiments}

Section~\ref{sec:exp-e1} shows the drift that Theorem~\ref{thm:certification}(b)
bounds; Sections~\ref{sec:exp-protocol} and~\ref{sec:exp-price} test the
theory's predictions on trained models.
Theorem~\ref{thm:invisibility} is an exact construction and is verified by
script rather than by training.  Uncertainty is mean $\pm$ standard error over
seeds or folds with $n$ stated, and every number is transcribed from the
per-run files.  Appendix~\ref{app:experiments} gives the full audit record, including the
precommitted readings and every disclosed departure from them.

\subsection{Agreement drifts when a trained VAE generates}
\label{sec:exp-e1}

Figure~\ref{fig:concept} (bottom) sketches agreement that holds on encoded
data and drifts when the model generates; a trained VAE shows this drift.  We audit five convolutional VAEs trained on
MNIST ($d=10$, $K=10$, one per seed).  $\cenc$ is a diagonal
Gaussian mixture on encoder means and $\cdec$ is the Bernoulli--Bregman
nearest-prototype rule on decoded outputs, both in their own indices, read
on $10{,}000$ test images with $S=16$ posterior and $20{,}000$ prior draws.
Agreement survives posterior sampling but degrades under the prior:
$\Aq=0.770\pm0.009$ versus $\Ap=0.525\pm0.007$ ($n=5$).  A task probe drops
in step, from $0.700\pm0.029$ to $0.390\pm0.019$, and a label-free prototype
control preserves the drop ($0.696\pm0.029\to0.375\pm0.027$), so the drift is
not an artifact of the label map.  The same drift appears in probabilistic
PCA, where every quantity is exact and no estimator can be blamed:
$\mathcal A[\qbar]-\Ap=+0.128$ to $+0.338$ across three code maps at $d=10$.

\subsection{Certificates fire where the theory says they can}
\label{sec:exp-protocol}

\textbf{The reconstruction side certifies trained encoders at useful
agreement.}  Since $\min_aB^*(a)=B^*(0.5)=0.1438$, any budget up to $0.1438$ certifies at
every agreement.  A trained mean-field MLP encoder on probabilistic PCA
($d=10$, canonical basis) has $\bar\Delta=0.088$ nats and certifies unconditionally.
Swept over latent dimension, the same certificate fires in $75/75$ cells at
$d\in\{2,5,10,20,32\}$ and agreement $0.76$--$0.97$, against $0/75$ for a trained \emph{linear}
encoder, with widths growing $0.106\to0.487$ and crossing $w^*$ just above
$d=32$, and every band contains the agreement under the exact posterior
(Figure~\ref{fig:hero}b, Appendix~\ref{app:nativerank}).  This side needs a
certified \emph{upper} bound on $\log p(x)$, which tractable models provide.

\textbf{On the rate side, licensing decides.}  Lemma~\ref{lem:budgets}
predicts that unlicensed plug-in budgets fail, and they do: a grid plug-in
certifies $20/20$ checkpoints it has no right to, and a moment-matched plug-in
excludes the truth in $30/30$ conditions.  The licensed mutual-information
budget of Section~\ref{sec:budgets} is sharply tighter than the best previously
available licensed budget \citep{kolchinsky2017estimating}: on six SVHN models,
$B_{\mathrm{MI}}=0.519$--$0.606$ nats against $1.467$--$1.555$.  By
Lemma~\ref{lem:width} the previous budget clears the necessary condition on
$0/6$ models: it admits no band narrower than $0.769$, whatever the agreement
and however the estimator is refined.  Ours lowers that floor to $\le0.455$ and clears the necessary
condition for half width on $6/6$.  On one of the six models, SVHN $d=32$, a nonlinear convolutional VAE, the
rate-side certificate fires (Appendix~\ref{app:record}).  The binary-KL form
matters: at that model's budget and point-estimate agreement, Pinsker gives
width $0.553$ and Bretagnolle--Huber $0.679$, both above $w^*$, where
binary~KL gives $0.487$.

\subsection{The price and the floor, measured}
\label{sec:exp-price}
\label{sec:exp-scale}

\textbf{A change of axes moves the budget and not the quantity.}  On
linear-Gaussian models, where Theorem~\ref{thm:price} is exact, a fixed DCT
rotation of the same model moves the mean-field budget from $0.000$ to $0.474$
nats and voids the reconstruction-side certificate for all three map
constructions, while the estimands move by at most $8\times10^{-4}$.  The
$O(d)$-closed full-covariance family is unaffected, and a trained mean-field
MLP encoder pays a coordinate price too, $0.088\to0.673$ nats
(Table~\ref{tab:price}, Appendix~\ref{app:price-models}).

\textbf{A change of code map moves the quantity and not the budget.}  This is
the dual in Theorem~\ref{thm:price}(c).  Across $240$ trained VAEs, changing
the declared map on the same checkpoint moves $\Aq$ by a median $0.565$ over
four declared maps ($239/240$ above $0.1$), while the rate budget, a
functional of $\qbar$ alone, stays where it was.  Agreement should therefore
always be reported with its map.

\textbf{The dataset floor predicts which models cannot be certified.}  We
trained $240$ fresh VAEs on FashionMNIST, MNIST, SVHN and CIFAR-10, up to the
full training split, with $d\in\{16,32,64\}$, $K\in\{10,32\}$, KL weights
$\beta$ up to $64$ and three seeds, and fixed every reading beforehand
(Appendix~\ref{app:scale}).  Theorem~\ref{thm:price}(e) forbids certification
once $\Rbar>\log N+\log2$.  That inequality holds for $66$ of the $240$
models, and none of them certifies (Figure~\ref{fig:hero}c).  Which side of the line a family falls on is not a matter of latent
dimension: on FashionMNIST $\Rbar-\log N$ reaches $+20.0$ nats and $36$ models
are excluded outright, while on SVHN it is negative throughout, $-9.7$ to
$-2.9$, and none is.  The floor is also attained, not merely respected: on six further FashionMNIST models (Appendix~\ref{app:loosest}), the measured $\hat I$ sits within
$0.01$ nats of $\log N$, so $\KL(\qbar\Vert p)=\Rbar-\log N$ to that accuracy.

\section{Related work}
\label{sec:related}

Mismatched decoding \citep{scarlett2020mismatched} and noisy-channel
quantization \citep{farvardin1990vq} name this failure mode for external
channels; we move it inside one VAE, between its encoder side and its decoder
side (VQ-VAE, \citealp{vandenoord2017neural}, is the constrained endpoint).
Standard diagnostics all sit inside Definition~\ref{def:package}: the ELBO
and IWAE \citep{kingma2014auto,rezende2014stochastic,burda2016importance},
rate--distortion decompositions \citep{alemi2018fixing}, disentanglement
scores \citep{higgins2017beta,chen2018isolating}, information-theoretic
summaries \citep{tishby2000information,cover2006elements,polyanskiy2024information},
and the mismatch between prior and aggregate posterior
\citep{hoffman2016elbo,tomczak2018vampprior,bauer2019resampled}.  Some of
these couple encoder and decoder, but only at the continuous level, so by
Theorem~\ref{thm:invisibility} none determines agreement between the codes.
Identifiability theory
\citep{hyvarinen2019nonlinear,khemakhem2020ivae,locatello2019challenging}
states a related indeterminacy asymptotically, and
Theorem~\ref{thm:invisibility}(II) is its finite-population complement.
Representation alignment
\citep{moschella2022relative,fumero2024latentfunctional,fumero2026latentdynamics,huh2024platonic}
compares different models, whereas agreement is measured inside one model
under its declared code maps.  A prior preprint \citep{hayashi2026codebook} already contains the coupled
channel, Theorem~\ref{thm:certification}(a) with its decomposition, and a
narrower impossibility result over code histograms, their entropies, occupied codes and $I(x;z)$,
and notes that a matched score can hide a codebook read entirely out of place.
Definition~\ref{def:package} closes that case; the generative side, and with
it the two-sided certificate, the price, the admission budget, the protocol
and all experiments are new.  Training-time consistency fixes
\citep{cemgil2020avae,wu2020coupled} and detection of collapse on the
marginal side \citep{dang2024posterior} neither certify nor measure the
coupled channel.  Decoder-side map geometry follows
\citet{banerjee2005clustering,nielsen2010bregman}, and certify-or-abstain,
standard in conformal and selective prediction
\citep{vovk2005algorithmic,geifman2017selective}, appears here with an exactly
characterized abstention region.

\section{Conclusion}
\label{sec:limitations}

A VAE can reconstruct every input perfectly, match its twin on every item of
a standard report, and still read none of its codes back as written.  The
ELBO's own terms, read as budgets, bound how far a measured agreement can
move, at a price and a floor the theory states in advance, and trained models
behave as predicted.  The reconstruction side needs a certified upper bound
on $\log p(x)$, so outside tractable models the protocol runs on the rate
side, where certifying a working codebook calls for a prior close to the aggregate posterior (Appendix~\ref{app:window}).  The practical rule is
short: log the joint code table in the pipeline's own indices, report
agreement with its code map and a mean-field budget with its axes, and use
Certify-or-Measure to decide whether to certify or to measure.

\enlargethispage{\baselineskip}
\paragraph{Reproducibility statement.}
All theorem-level claims are accompanied by numerical verification scripts
with saved logs (grid witness at $K\in\{2,3,5,8\}$; interleaved witness at
all $(K,m)\in\{2,3\}\times\{3,4,5\}$ with a radius-robustness check;
closed-form-versus-numeric optimality, exact supremum, degenerate spectra,
width-lemma bounds, and transport invariance for
Theorem~\ref{thm:price}).  The replication grid and the
protocol run were preregistered (model sets, readings, and verdict criteria
fixed before execution; the readings are reproduced verbatim in
Figure~\ref{fig:precommit}); the three places where execution departed from the
precommit are disclosed in Appendix~\ref{app:experiments}, with the
sensitivity of every affected conclusion measured rather than asserted.
Per-run CSVs, precommit documents, deviation disclosures, and the analysis
ledger are included in the supplementary material (on arXiv, the
ancillary files), and every number in this paper is
transcribed from those primary files.

\paragraph{Ethics statement.}
This work is a theoretical and empirical audit of representation
consistency in generative models trained on standard public benchmarks
(MNIST, FashionMNIST, SVHN, CIFAR-10); it involves no human subjects, no personal
data, and no foreseeable direct dual-use concern.  Its intended effect is
to make silent failures of discrete latent interfaces measurable and
disclosed rather than assumed away, which we view as ethically positive
for downstream deployments that read codes from generative models.

\paragraph{AI use statement.}
Large language models were used as writing, coding and checking assistants
throughout: drafting and revising prose; writing, refactoring and running
analysis and verification scripts; maintaining the results ledgers; and
adversarially checking derivations and stated numbers against the primary
result files (one such check found the sign error in the clip correction
disclosed in Appendix~\ref{app:experiments}).  All
theorems, proofs, experimental designs, and conclusions were specified,
verified, and approved by the author; every number in the paper is transcribed from the primary logs, the
verification script included in the supplement automatically re-checks most
of them against those files, and the author takes full responsibility for
the content.

\bibliography{references}

\begin{thebibliography}{45}
\providecommand{\natexlab}[1]{#1}
\providecommand{\url}[1]{\texttt{#1}}
\expandafter\ifx\csname urlstyle\endcsname\relax
  \providecommand{\doi}[1]{doi: #1}\else
  \providecommand{\doi}{doi: \begingroup \urlstyle{rm}\Url}\fi

\bibitem[Alemi et~al.(2018)]{alemi2018fixing}
Alexander~A. Alemi et~al.
\newblock Fixing a broken {ELBO}.
\newblock In \emph{International Conference on Machine Learning (ICML)}, 2018.

\bibitem[Audibert et~al.(2009)Audibert, Munos, and
  Szepesv{\'a}ri]{audibert2009exploration}
Jean-Yves Audibert, R{\'e}mi Munos, and Csaba Szepesv{\'a}ri.
\newblock Exploration-exploitation tradeoff using variance estimates in
  multi-armed bandits.
\newblock \emph{Theoretical Computer Science}, 410\penalty0 (19):\penalty0
  1876--1902, 2009.

\bibitem[Banerjee et~al.(2005)Banerjee, Merugu, Dhillon, and
  Ghosh]{banerjee2005clustering}
Arindam Banerjee, Srujana Merugu, Inderjit~S. Dhillon, and Joydeep Ghosh.
\newblock Clustering with {B}regman divergences.
\newblock \emph{Journal of Machine Learning Research}, 6:\penalty0 1705--1749,
  2005.

\bibitem[Bauer \& Mnih(2019)Bauer and Mnih]{bauer2019resampled}
Matthias Bauer and Andriy Mnih.
\newblock Resampled priors for variational autoencoders.
\newblock In \emph{International Conference on Artificial Intelligence and
  Statistics (AISTATS)}, 2019.

\bibitem[Bi{\'n}kowski et~al.(2018)Bi{\'n}kowski, Sutherland, Arbel, and
  Gretton]{binkowski2018demystifying}
Miko{\l}aj Bi{\'n}kowski, Danica~J. Sutherland, Michael Arbel, and Arthur
  Gretton.
\newblock Demystifying {MMD} {GANs}.
\newblock In \emph{International Conference on Learning Representations}, 2018.

\bibitem[Bishop(2006)]{bishop2006pattern}
Christopher~M. Bishop.
\newblock \emph{Pattern Recognition and Machine Learning}.
\newblock Springer, 2006.

\bibitem[Burda et~al.(2016)Burda, Grosse, and
  Salakhutdinov]{burda2016importance}
Yuri Burda, Roger Grosse, and Ruslan Salakhutdinov.
\newblock Importance weighted autoencoders.
\newblock In \emph{International Conference on Learning Representations
  (ICLR)}, 2016.

\bibitem[Cemgil et~al.(2020)Cemgil, Ghaisas, Dvijotham, Gowal, and
  Kohli]{cemgil2020avae}
A.~Taylan Cemgil, Sumedh Ghaisas, Krishnamurthy Dvijotham, Sven Gowal, and
  Pushmeet Kohli.
\newblock The autoencoding variational autoencoder.
\newblock In \emph{Advances in Neural Information Processing Systems
  (NeurIPS)}, 2020.

\bibitem[Chen et~al.(2018)Chen, Li, Grosse, and Duvenaud]{chen2018isolating}
Ricky T.~Q. Chen, Xuechen Li, Roger~B. Grosse, and David~K. Duvenaud.
\newblock Isolating sources of disentanglement in variational autoencoders.
\newblock In \emph{Advances in Neural Information Processing Systems
  (NeurIPS)}, 2018.

\bibitem[Cover \& Thomas(2006)Cover and Thomas]{cover2006elements}
Thomas~M. Cover and Joy~A. Thomas.
\newblock \emph{Elements of Information Theory}.
\newblock Wiley, 2nd edition, 2006.

\bibitem[Cremer et~al.(2018)Cremer, Li, and Duvenaud]{cremer2018inference}
Chris Cremer, Xuechen Li, and David Duvenaud.
\newblock Inference suboptimality in variational autoencoders.
\newblock In \emph{International Conference on Machine Learning}, 2018.

\bibitem[Dang et~al.(2024)Dang, Tran, Nguyen, and Ho]{dang2024posterior}
H.~Dang, T.~Tran, T.~Nguyen, and N.~Ho.
\newblock Beyond vanilla variational autoencoders: Detecting posterior collapse
  in conditional and hierarchical variational autoencoders.
\newblock In \emph{International Conference on Learning Representations
  (ICLR)}, 2024.

\bibitem[Dvoretzky et~al.(1956)Dvoretzky, Kiefer, and
  Wolfowitz]{dvoretzky1956asymptotic}
Aryeh Dvoretzky, Jack Kiefer, and Jacob Wolfowitz.
\newblock Asymptotic minimax character of the sample distribution function and
  of the classical multinomial estimator.
\newblock \emph{Annals of Mathematical Statistics}, 27\penalty0 (3):\penalty0
  642--669, 1956.

\bibitem[Farvardin(1990)]{farvardin1990vq}
Nariman Farvardin.
\newblock A study of vector quantization for noisy channels.
\newblock \emph{IEEE Transactions on Information Theory}, 36\penalty0
  (4):\penalty0 799--809, 1990.

\bibitem[Fumero et~al.(2024)Fumero, Pegoraro, Maiorca, Locatello, and
  Rodol{\`a}]{fumero2024latentfunctional}
Marco Fumero, Marco Pegoraro, Valentino Maiorca, Francesco Locatello, and
  Emanuele Rodol{\`a}.
\newblock Latent functional maps: A spectral framework for representation
  alignment.
\newblock In \emph{Advances in Neural Information Processing Systems
  (NeurIPS)}, 2024.

\bibitem[Fumero et~al.(2026)Fumero, Moschella, Rodol{\`a}, and
  Locatello]{fumero2026latentdynamics}
Marco Fumero, Luca Moschella, Emanuele Rodol{\`a}, and Francesco Locatello.
\newblock Navigating the latent space dynamics of neural models.
\newblock \emph{arXiv:2505.22785}, 2026.

\bibitem[Geifman \& El-Yaniv(2017)Geifman and El-Yaniv]{geifman2017selective}
Yonatan Geifman and Ran El-Yaniv.
\newblock Selective classification for deep neural networks.
\newblock In \emph{Advances in Neural Information Processing Systems}, 2017.

\bibitem[Hayashi(2026)]{hayashi2026codebook}
Yusuke Hayashi.
\newblock Lost and found in translation: Variational diagnostics for neural
  codebook channels.
\newblock \emph{arXiv:2605.18846}, 2026.

\bibitem[Heusel et~al.(2017)Heusel, Ramsauer, Unterthiner, Nessler, and
  Hochreiter]{heusel2017gans}
Martin Heusel, Hubert Ramsauer, Thomas Unterthiner, Bernhard Nessler, and Sepp
  Hochreiter.
\newblock {GANs} trained by a two time-scale update rule converge to a local
  {N}ash equilibrium.
\newblock In \emph{Advances in Neural Information Processing Systems}, 2017.

\bibitem[Higgins et~al.(2017)]{higgins2017beta}
Irina Higgins et~al.
\newblock beta-{VAE}: Learning basic visual concepts with a constrained
  variational framework.
\newblock In \emph{International Conference on Learning Representations
  (ICLR)}, 2017.

\bibitem[Hoffman \& Johnson(2016)Hoffman and Johnson]{hoffman2016elbo}
Matthew~D. Hoffman and Matthew~J. Johnson.
\newblock {ELBO} surgery: yet another way to carve up the variational evidence
  lower bound.
\newblock In \emph{NeurIPS Workshop on Advances in Approximate Bayesian
  Inference}, 2016.

\bibitem[Horn \& Johnson(2013)Horn and Johnson]{horn2013matrix}
Roger~A. Horn and Charles~R. Johnson.
\newblock \emph{Matrix Analysis}.
\newblock Cambridge University Press, 2nd edition, 2013.

\bibitem[Huh et~al.(2024)Huh, Cheung, Wang, and Isola]{huh2024platonic}
Minyoung Huh, Brian Cheung, Tongzhou Wang, and Phillip Isola.
\newblock Position: The platonic representation hypothesis.
\newblock In \emph{International Conference on Machine Learning (ICML)}, 2024.

\bibitem[Hyv{\"a}rinen et~al.(2019)Hyv{\"a}rinen, Sasaki, and
  Turner]{hyvarinen2019nonlinear}
Aapo Hyv{\"a}rinen, Hiroaki Sasaki, and Richard~E. Turner.
\newblock Nonlinear {ICA} using auxiliary variables and generalized contrastive
  learning.
\newblock In \emph{International Conference on Artificial Intelligence and
  Statistics (AISTATS)}, 2019.

\bibitem[Khemakhem et~al.(2020)Khemakhem, Kingma, Monti, and
  Hyv{\"a}rinen]{khemakhem2020ivae}
Ilyes Khemakhem, Diederik Kingma, Ricardo Monti, and Aapo Hyv{\"a}rinen.
\newblock Variational autoencoders and nonlinear {ICA}: a unifying framework.
\newblock In \emph{International Conference on Artificial Intelligence and
  Statistics (AISTATS)}, 2020.

\bibitem[Kingma \& Welling(2014)Kingma and Welling]{kingma2014auto}
Diederik~P. Kingma and Max Welling.
\newblock Auto-encoding variational {B}ayes.
\newblock In \emph{International Conference on Learning Representations
  (ICLR)}, 2014.

\bibitem[Kolchinsky \& Tracey(2017)Kolchinsky and
  Tracey]{kolchinsky2017estimating}
Artemy Kolchinsky and Brendan~D. Tracey.
\newblock Estimating mixture entropy with pairwise distances.
\newblock \emph{Entropy}, 19\penalty0 (7):\penalty0 361, 2017.

\bibitem[Loaiza-Ganem \& Cunningham(2019)Loaiza-Ganem and
  Cunningham]{loaiza2019continuous}
Gabriel Loaiza-Ganem and John~P. Cunningham.
\newblock The continuous {B}ernoulli: fixing a pervasive error in variational
  autoencoders.
\newblock In \emph{Advances in Neural Information Processing Systems
  (NeurIPS)}, 2019.

\bibitem[Locatello et~al.(2019)Locatello, Bauer, Lucic, R{\"a}tsch, Gelly,
  Sch{\"o}lkopf, and Bachem]{locatello2019challenging}
Francesco Locatello, Stefan Bauer, Mario Lucic, Gunnar R{\"a}tsch, Sylvain
  Gelly, Bernhard Sch{\"o}lkopf, and Olivier Bachem.
\newblock Challenging common assumptions in the unsupervised learning of
  disentangled representations.
\newblock In \emph{International Conference on Machine Learning (ICML)}, 2019.

\bibitem[Marshall et~al.(2011)Marshall, Olkin, and
  Arnold]{marshall2011inequalities}
Albert~W. Marshall, Ingram Olkin, and Barry~C. Arnold.
\newblock Inequalities: Theory of majorization and its applications.
\newblock 2011.

\bibitem[Maurer \& Pontil(2009)Maurer and Pontil]{maurer2009empirical}
Andreas Maurer and Massimiliano Pontil.
\newblock Empirical {B}ernstein bounds and sample variance penalization.
\newblock In \emph{Conference on Learning Theory}, 2009.

\bibitem[Moschella et~al.(2023)Moschella, Maiorca, Fumero, Norelli, Locatello,
  and Rodol{\`a}]{moschella2022relative}
Luca Moschella, Valentino Maiorca, Marco Fumero, Antonio Norelli, Francesco
  Locatello, and Emanuele Rodol{\`a}.
\newblock Relative representations enable zero-shot latent space communication.
\newblock In \emph{International Conference on Learning Representations
  (ICLR)}, 2023.

\bibitem[Nielsen et~al.(2010)Nielsen, Boissonnat, and Nock]{nielsen2010bregman}
Frank Nielsen, Jean-Daniel Boissonnat, and Richard Nock.
\newblock Bregman {V}oronoi diagrams: properties, algorithms and applications.
\newblock \emph{Discrete \& Computational Geometry}, 44\penalty0 (2):\penalty0
  281--307, 2010.

\bibitem[Polyanskiy \& Wu(2025)Polyanskiy and Wu]{polyanskiy2024information}
Yury Polyanskiy and Yihong Wu.
\newblock \emph{Information Theory: From Coding to Learning}.
\newblock Cambridge University Press, 2025.

\bibitem[Poole et~al.(2019)Poole, Ozair, van~den Oord, Alemi, and
  Tucker]{poole2019variational}
Ben Poole, Sherjil Ozair, A{\"a}ron van~den Oord, Alexander~A. Alemi, and
  George Tucker.
\newblock On variational bounds of mutual information.
\newblock In \emph{International Conference on Machine Learning}, 2019.

\bibitem[Rezende et~al.(2014)Rezende, Mohamed, and
  Wierstra]{rezende2014stochastic}
Danilo~Jimenez Rezende, Shakir Mohamed, and Daan Wierstra.
\newblock Stochastic backpropagation and approximate inference in deep
  generative models.
\newblock In \emph{International Conference on Machine Learning (ICML)}, 2014.

\bibitem[Scarlett et~al.(2020)Scarlett, Guill{\'e}n~i F{\`a}bregas,
  Somekh-Baruch, and Martinez]{scarlett2020mismatched}
Jonathan Scarlett, Albert Guill{\'e}n~i F{\`a}bregas, Anelia Somekh-Baruch, and
  Alfonso Martinez.
\newblock Information-theoretic foundations of mismatched decoding.
\newblock \emph{Foundations and Trends in Communications and Information
  Theory}, 17\penalty0 (2--3):\penalty0 149--401, 2020.

\bibitem[Shannon(1948)]{shannon1948}
Claude~E. Shannon.
\newblock A mathematical theory of communication.
\newblock \emph{Bell System Technical Journal}, 27\penalty0 (3):\penalty0
  379--423, 1948.
\newblock \doi{10.1002/j.1538-7305.1948.tb01338.x}.

\bibitem[Tishby et~al.(2000)Tishby, Pereira, and Bialek]{tishby2000information}
Naftali Tishby, Fernando~C. Pereira, and William Bialek.
\newblock The information bottleneck method.
\newblock \emph{arXiv:physics/0004057}, 2000.

\bibitem[Tomczak \& Welling(2018)Tomczak and Welling]{tomczak2018vampprior}
Jakub~M. Tomczak and Max Welling.
\newblock {VAE} with a {V}amp{P}rior.
\newblock In \emph{International Conference on Artificial Intelligence and
  Statistics (AISTATS)}, 2018.

\bibitem[van~den Oord et~al.(2017)van~den Oord, Vinyals, and
  Kavukcuoglu]{vandenoord2017neural}
Aaron van~den Oord, Oriol Vinyals, and Koray Kavukcuoglu.
\newblock Neural discrete representation learning.
\newblock In \emph{Advances in Neural Information Processing Systems
  (NeurIPS)}, 2017.

\bibitem[van~den Oord et~al.(2018)van~den Oord, Li, and Vinyals]{oord2018cpc}
A{\"a}ron van~den Oord, Yazhe Li, and Oriol Vinyals.
\newblock Representation learning with contrastive predictive coding.
\newblock \emph{arXiv preprint arXiv:1807.03748}, 2018.

\bibitem[Vovk et~al.(2005)Vovk, Gammerman, and Shafer]{vovk2005algorithmic}
Vladimir Vovk, Alexander Gammerman, and Glenn Shafer.
\newblock \emph{Algorithmic Learning in a Random World}.
\newblock Springer, 2005.

\bibitem[Weaver(1949)]{weaver1949}
Warren Weaver.
\newblock Recent contributions to the mathematical theory of communication.
\newblock In \emph{The Mathematical Theory of Communication}. University of
  Illinois Press, Urbana, IL, 1949.

\bibitem[Wu et~al.(2020)Wu, Wang, and Wang]{wu2020coupled}
Chen Wu, Prince~Zizhuang Wang, and William~Yang Wang.
\newblock On the encoder--decoder incompatibility in variational text modeling
  and beyond.
\newblock In \emph{Proceedings of the 58th Annual Meeting of the Association
  for Computational Linguistics}, pp.\  3449--3464, 2020.

\end{thebibliography}
\bibliographystyle{iclr2027_conference}

\appendix

\section{Proofs for Theorem~\ref{thm:invisibility}}
\label{app:proofs}

\subsection{The grid witness, in full}
\label{app:grid}

Fix $K\ge2$ and a derangement $\sigma$ of $[K]$.  Latent space
$\cZ=[0,1)^2$ with uniform (Lebesgue) prior; cells
$G_{ij}=[\tfrac{i-1}K,\tfrac iK)\times[\tfrac{j-1}K,\tfrac jK)$, each of
prior mass $1/K^2$; fixed maps $\cenc(z)=\lceil Kz_1\rceil$,
$\cdec(z)=\lceil Kz_2\rceil$.  Data
$\cX=\{v_1,\dots,v_K\}\cup J$ with $J$ a pool of $K^2-K$ distinct junk
atoms; $\mathcal D=\mathrm{Unif}\{v_1,\dots,v_K\}$.  Configuration $C^1$:
the decoder emits $v_i$ deterministically on $G_{ii}$ and assigns the junk
atoms bijectively to the off-diagonal cells; encoder
$q^1(\cdot\mid v_i)=\mathrm{Unif}(G_{ii})$.  Configuration $C^2$: the
decoder emits $v_i$ on $G_{i\sigma(i)}$ and the same junk pool bijectively
on the complementary cells; $q^2(\cdot\mid v_i)=\mathrm{Unif}(G_{i\sigma(i)})$.

\emph{(m1)--(m3).}  For $x=v_i$ the posterior is uniform on the unique
cell where the decoder emits $v_i$, so reconstruction
$\E_q\log p(x\mid z)=\log1=0$; rate
$\KL(\mathrm{Unif}(\text{cell})\Vert\mathrm{Unif}([0,1)^2))=\log K^2$;
ELBO $=-\log K^2$.  The marginal likelihood is the prior mass of the
emitting cell: $p(v_i)=1/K^2$, so $\log p(v_i)=\mathrm{ELBO}(v_i)$ and the
gap is $0$: inference is exact ($q(\cdot\mid v_i)$ \emph{is} the true
posterior, uniform on the emitting cell).  All values coincide across
$C^1,C^2$.  As a distribution, $p(x)$ is uniform over the same $K^2$ atoms
(each cell emits one atom of mass $1/K^2$) in both configurations.

\emph{(m4).}  $\qbar$ has density $K$ on the $K$ occupied cells, so
$h(\qbar)=-\log K$, while each $q(\cdot\mid v_i)$ is uniform on one cell of
area $K^{-2}$, so $\E_xh(q(\cdot\mid x))=-\log K^2$; hence
$I(x;z)=h(\qbar)-\E_xh(q(\cdot\mid x))=\log K$ in both.  Posterior means are cell centers; coordinate~1 means form
the multiset $\{(i-\tfrac12)/K\}$ in both; coordinate~2 means form
$\{(i-\tfrac12)/K\}$ versus $\{(\sigma(i)-\tfrac12)/K\}$---equal as
multisets since $\sigma$ is a bijection---so each per-dimension
$\mathrm{AU}_k$ coincides.  (The \emph{cross-covariance} of the mean pairs
differs; that leak is what Part~I$'$ removes.)

\emph{(m5)--(m7).}  The joint tables under $\qbar$ are
$T^1=\tfrac1K\mathrm{Id}$ and $T^2=\tfrac1KP_\sigma$: both have uniform
row and column sums (identical histograms, entropies, occupied counts) and
are related by a column permutation, $T^2=T^1P_\sigma$, so \emph{every}
$F$ with $F(P_1TP_2^{\top})=F(T)$ coincides on them---this is the exact reason
(m6) cannot separate the pair, and it is stronger than equality of entry
multisets, which alone would not suffice.  In particular
$I(\cenc;\cdec)=\log K$ in both and the permutation-matched agreement is
$1$ in both.  Under the prior the code pair is uniform on
$[K]^2$, so $\Ap=1/K$ in both.  Yet
$\mathcal A(C^1)=\qbar^1(\bigcup_iG_{ii})=1$ and $\mathcal A(C^2)=0$.
\hfill$\square$

\paragraph{Smoothing remark.}
Embed $\cX$ mutually equidistantly (one-hot in $\R^{K^2}$ suffices) and
replace deterministic emission by $p(x\mid z)=N(x;x_{\mathrm{cell}}(z),
\varepsilon^2I)$.  By equidistance, the likelihood of any atom under any
non-emitting cell is a shared constant, and the entire package equality is
preserved exactly for every $\varepsilon>0$, with three shared shifts:
reconstruction becomes a common constant, ELBO and $\log p(x)$ shift with
it, and the common gap becomes positive; the rate is unchanged.  The
equality is also exact under a generic (non-equidistant) embedding: the
encoders are unchanged, and in both configurations every cell emits one of the
same $K^2$ atoms with weight $1/K^2$, so
$p(x)=K^{-2}\sum_{\text{cells}}N(x;x_{\mathrm{cell}},\varepsilon^2I)$ is the
same distribution, reconstruction is the same constant
$\log N(v_i;v_i,\varepsilon^2I)$, and the rate, ELBO, $\log p(x)$ and gap
coincide per example; (m4)--(m7) do not involve the emission.  Equidistance
only makes the shared values explicit.

\subsection{The interleaved product witness}
\label{app:interleaved}

\emph{Overview (sketch of Part~I$'$).}
The grid pair still leaks through one non-package statistic: per-example
posterior \emph{means} encode both codes.  The repair takes
$\cZ=\R^2\times\R^2$ with a product of uniform annuli: $\cenc$ reads a
coarse angular sector ($K$ of them) of plane $\mathsf A$, and $\cdec$ an
$m$-fold rotationally symmetric \emph{interleaved class} of plane
$\mathsf B$.  Classes are invariant under rotation by $2\pi/m$, so their
moment tensors of degree $<m$ are isotropic and class-independent: the
pairing $\sigma$ is invisible to \emph{all} aggregate moments of
$\mathsf B$-degree $<m$, and per-example means are identical outright.  At
degree exactly $m$ an $m$-th harmonic carries the class index, and
explicit cross moments separate the pair (numerically at $K=3$, $m=3$:
the degree-$3$ moment
$\E[z_{\mathsf A,2}z_{\mathsf B,1}^2z_{\mathsf B,2}]$ equals $0.1040$
versus $-0.0520$)---as something must, since agreement is a
$\qbar$-functional.

The engine is a harmonic-blindness lemma (Lemma~\ref{lem:harmonic},
stated and proved next), a one-line
consequence of the circular-harmonic expansion: rotation multiplies the
$k$-th harmonic by $e^{ik\omega}$, and invariance under rotation by
$2\pi/m$ kills every harmonic with $m\nmid k$, leaving only the
rotation-invariant $k=0$ term below degree $m$.  Part~(II) is proved in
Appendix~\ref{app:orbit}.

\begin{lemma}[Harmonic blindness]\label{lem:harmonic}
Let $\mu_0$ be a distribution on $\R^2$ invariant under rotation by
$2\pi/m$, and $\mu_j=(R_{\theta_j})_{\#}\mu_0$ its rotations.  Then all
moment tensors of $\mu_j$ of total degree $t<m$ are independent of $j$; if
$\E_{\mu_0}[r^m e^{im\varphi}]\ne0$, the degree-$m$ tensors depend on $j$
through the phase $e^{im\theta_j}$.
\end{lemma}

\begin{proof}[Proof of Lemma~\ref{lem:harmonic}]
In polar form a degree-$t$ monomial moment is a linear combination of
circular harmonics $\E[r^te^{ik\varphi}]$ with $|k|\le t$ and
$k\equiv t\pmod2$.  Rotation by $\omega$ multiplies the $k$-th harmonic by
$e^{ik\omega}$; invariance of $\mu_0$ under $\omega=2\pi/m$ forces
$\E_{\mu_0}[r^te^{ik\varphi}]=0$ unless $m\mid k$.  For $t<m$ only $k=0$
survives, and $k=0$ harmonics are rotation-invariant, hence equal for all
$\mu_j$.  At $t=m$ the $k=\pm m$ harmonics survive and transform by
$e^{\pm im\theta_j}$.
\end{proof}

\emph{Construction.}  Take $\cZ=\R^2\times\R^2$ with a product of uniform
annuli.  $\cenc$ reads the coarse angular sector ($K$ of them) of plane
$\mathsf A$.  On plane $\mathsf B$, partition into $Km$ fine sectors and
let class $j$ be the union of the $m$ sectors congruent to $j$ mod
$K$---an $m$-fold rotationally symmetric set of zero mean; $\cdec$ reads
the class.  The two configurations couple sectors to classes by the
identity and by the derangement $\sigma$, exactly as the grid couples rows
to columns.  The package items
repeat the grid computation verbatim, with joint cells
$O_{ij}=\text{sector}_{\mathsf A}(i)\times\text{class}_{\mathsf B}(j)$ (prior mass $1/K^2$
each; the cells are congruent, related by rotations of the two planes, so
rates and masses match).  For the new items:

\emph{(N1).}  $\E[z\mid v_i]=(\mu_{\mathsf A}(i),\,\mu_{\mathsf B}(\text{class}))$ where
$\mu_{\mathsf B}$ of any class vanishes: a class is invariant under rotation by
$2\pi/m$ ($m\ge2$), and the only vector fixed by a nontrivial rotation is
$0$.  Hence $\E[z\mid v_i]=(\mu_{\mathsf A}(i),0,0)$ in both configurations,
example by example.

\emph{(N2).}  A mixed moment factorizes over the product cell:
$\E_{\qbar}[z_{\mathsf A}^{a} z_{\mathsf B}^{b}]=\tfrac1K\sum_i
\E[z_{\mathsf A}^{a}\mid\text{sector}_{\mathsf A}(i)]\cdot
\E[z_{\mathsf B}^{b}\mid\text{class}(c(i))]$ with $c=\mathrm{id}$ or $\sigma$.  By
Lemma~\ref{lem:harmonic}, for $|b|<m$ the class-conditional factor is
class-independent, so the sum is identical for both couplings.

\emph{(N3).}  At $|b|=m$ the $k=\pm m$ harmonic of the class survives
and carries the class phase $e^{im\theta_j}$ with $\theta_j=2\pi j/(Km)$,
i.e.\ the $K$-th root of unity $e^{2\pi ij/K}$: the factor now depends on
$j$, and pairing it with any class-distinguishing $\mathsf A$-factor produces an
explicit separating moment; the verification log exhibits one for every
$(K,m)$ in the grid.  Since $\mathcal A$ is a $\qbar$-functional,
\emph{some} statistic had to separate the pair; the construction pins the
minimal polynomial order at exactly $m$.  \hfill$\square$

\subsection{Scope of Theorem~\ref{thm:invisibility}}
\label{app:scope}

The theorem does not claim the following.
(i)~Part~I matches the full package at uniform occupied masses; the
full-package question at non-uniform masses is open---Proposition~\ref{prop:a4} gives the
sharp \emph{channel-level} statement there, and when $\max_i\nu_i>\frac12$
zero agreement is infeasible ($\mathcal A\ge2\max_i\nu_i-1$), leaving strict
separation only.
(ii)~The maps are fixed; identifiability for maps refit per configuration
is outside the theorem, and the witnesses are constructions---
no genericity is claimed: we exhibit constructions, not a sampling argument
about trained models.
(iii)~In Proposition~\ref{prop:structure} the $K=3$ witness matches (agreement, marginals,
entropies, MI) but is separated by, e.g., its maximal entry ($0.2333$
vs.\ $0.3070$) and Frobenius norm ($0.4397$ vs.\ $0.4515$): the claim is
that the scalar summaries the literature reports do not determine
structure, not that no table functional does---the table's placement does,
which is why reporting $T$ is the recommendation.
(iv)~All constructions are verified numerically to machine precision (grid
witness $K\in\{2,3,5,8\}$; interleaved witness for all
$(K,m)\in\{2,3\}\times\{3,4,5\}$ plus a radius-robustness check); scripts
and logs are in the supplement.

\subsection{Non-uniform marginals with matched mutual information}
\label{app:a4}

\begin{proposition}[Non-uniform marginals]\label{prop:a4}
For every probability vector $\nu$ with at least two positive entries there
exist coupled channels with index-wise identical encoder and decoder
marginals $\nu$ and exactly matched $I(\cenc;\cdec)$ whose agreements differ;
when $\max_i\nu_i\le\frac12$ the pair can be taken with $\mathcal A=0$ versus
$\mathcal A>\sum_i\nu_i^2$.  The feasible set of $\mathcal A$ over all
couplings with both marginals $\nu$ is exactly
$[\max(0,\,2\max_i\nu_i-1),\,1]$.
\end{proposition}

Let $\nu=(\nu_1,\dots,\nu_K)$ have at least two positive entries.  Define the
independence coupling $\Pi=\nu\nu^{\top}$
($\mathcal A(\Pi)=\sum_i\nu_i^2=:A_0$), the identity coupling
$P_{\mathrm{high}}=\mathrm{diag}(\nu)$ ($\mathcal A=1$), and any coupling
$P_{\mathrm{low}}$ with both marginals $\nu$ attaining the minimal diagonal
mass $A_{\min}=\max(0,2\max_i\nu_i-1)$.  Existence of $P_{\mathrm{low}}$:
for every coupling, $P(i,i)\ge \nu_i-\sum_{j\ne i}P(i,j)\ge2\nu_i-1$ and at
most one index can exceed $\tfrac12$, so $\mathcal A\ge A_{\min}$ always;
attainability for $\max_i\nu_i\le\tfrac12$ is Hall/max-flow feasibility of a
zero-diagonal doubly-$\nu$ table (feasible iff $\nu_i\le1-\nu_i$ for all $i$),
and for $\max_i\nu_i>\tfrac12$ by an explicit greedy construction.  This
proves the exact feasible interval
$[\max(0,2\max_i\nu_i-1),1]$ claimed in Proposition~\ref{prop:a4}.

Consider the segments $P_t=(1-t)\Pi+tP_{\mathrm{high}}$ and
$\bar P_s=(1-s)\Pi+sP_{\mathrm{low}}$.  (i)~Every $P_t,\bar P_s$ has both marginals
index-wise equal to $\nu$.  (ii)~With marginals fixed,
$I=2H(\nu)-H(\text{joint})$; joint entropy is strictly concave and each path
is affine and non-constant, so $I$ is convex, continuous, vanishes only at
$\Pi$, and is strictly increasing along each segment, ranging over
$[0,H(\nu)]$ on $\{P_t\}$ and $[0,I(P_{\mathrm{low}})]$ on $\{\bar P_s\}$ with
$I(P_{\mathrm{low}})>0$.  (iii)~Hence for every
$m\in(0,I(P_{\mathrm{low}})]$ there are unique $t^*,s^*$ with
$I(P_{t^*})=I(\bar P_{s^*})=m$, and since $\mathcal A$ is affine along each
path with slopes $1-A_0>0$ and $A_{\min}-A_0<0$,
$\mathcal A(P_{t^*})>A_0>\mathcal A(\bar P_{s^*})$, with gap
$t^*(1-A_0)+s^*(A_0-A_{\min})>0$.  When $\max_i\nu_i\le\tfrac12$, taking
$P_{\mathrm{low}}$ with zero diagonal and $m=I(P_{\mathrm{low}})$ gives
the zero-versus-positive form with
$\mathcal A(P_{t^*})>A_0=\sum_i\nu_i^2$.  \hfill$\square$

\emph{Scope discipline.}  For $m>I(P_{\mathrm{low}})$ no claim is made:
for $K=2$ and skewed $\nu$, marginals plus MI can \emph{determine}
$\mathcal A$ uniquely on part of the range, so the matched-MI window is
genuinely $(0,I(P_{\mathrm{low}})]$; the window is quantitatively small
for very skewed $\nu$---that is the coupling floor at work, and it is
consistent with, not contrary to, the certification narrative.  The pair
below settles whether the marginals can be made to match \emph{exactly}, and not
only after relabeling, at exactly equal mutual information, and settles it with an
explicit table rather than an existence claim.

\paragraph{An explicit instance: marginals matched exactly, no relabeling.}
Take $K=3$ and $\nu=(\tfrac12,0.3,0.2)$.  Row~$1$ carries mass $\tfrac12$ while
columns $2,3$ have total capacity $\tfrac12$, so the zero-diagonal coupling with
both marginals $\nu$ is \emph{unique}; solving $I(P_t)=I(P_{\mathrm{low}})$ along
$P_t=(1-t)\nu\nu^{\top}+t\,\mathrm{diag}(\nu)$ gives $t^*=0.866779$:
\[
P_{\mathrm{low}}=\begin{pmatrix}0&0.3&0.2\\0.3&0&0\\0.2&0&0\end{pmatrix},
\qquad
P_{t^*}=\begin{pmatrix}
0.466695&0.019983&0.013322\\
0.019983&0.272024&0.007993\\
0.013322&0.007993&0.178685\end{pmatrix}.
\]
Both tables have encoder \emph{and} decoder marginals index-wise equal to $\nu$,
and both have mutual information exactly $\log2=0.693147$ nats: whenever
$\nu_1=\tfrac12$ the off-diagonal terms cancel in $I=2H(\nu)-H(P)$, leaving
$I(P_{\mathrm{low}})=\log2$ in closed form.  Yet
$\mathcal A(P_{\mathrm{low}})=0$ exactly while $\mathcal A(P_{t^*})=0.917403$,
against $A_0=\sum_i\nu_i^2=0.38$ for the independent coupling.  The two tables
are not relabelings of one another---their entry multisets differ---so here the
marginals match exactly, and not only up to relabeling.  This is a
channel-level statement about two joint tables, which is the level at which the
question was posed; the full-package version at non-uniform masses remains open
(Appendix~\ref{app:scope}).

\subsection{Orbit form}
\label{app:orbit}

\emph{Full statement of Theorem~\ref{thm:invisibility}(II).}
Let $\tau$ be a prior-preserving bimeasurable bijection of $\cZ$ and
$C_\tau=(\mathcal D,\,\tau^{-1}_{\#}q,\,p(x\mid\tau\,\cdot),\,\cenc,\cdec)$.
Then (m1)--(m3) and $I(x;z)$ are exactly invariant for every $x$; the
per-dimension AU is additionally invariant whenever $\tau^{-1}$ restricted to
$\mathrm{supp}(\qbar)$ is a translation.  Agreement transforms as
$\mathcal A(C_\tau)=\qbar(\tau(\Agr))$ and changes iff $\tau$ effects a
net transport of $\qbar$-mass across the disagreement boundary.

\emph{Proof.}
For each $x$ the pair (encoder pushforward, decoder pullback) is a change
of variables: reconstruction integrands compose to the original,
$\KL$ is invariant under bimeasurable bijections applied to both
arguments (prior preserved), so (m1)--(m3) are pointwise invariant, and
$I(x;z)$ is invariant as a functional of the joint law.  The aggregate
posterior transforms as $\qbar_\tau=\tau^{-1}_{\#}\qbar$, so
\[
\mathcal A(C_\tau)=\qbar_\tau(\Agr)=\qbar(\tau(\Agr)),
\]
which differs from $\mathcal A=\qbar(\Agr)$ iff $\tau$ transports net
$\qbar$-mass across the boundary of $\Agr$.  (The orientation matters:
writing $\tau^{-1}(\Agr)$ here is wrong, though undetectable when $\tau$
is an involution; the verification suite checks a non-involutive
quarter-shift.)  For AU: the latent of $C_\tau$ is $\tau^{-1}(z)$ with $z\sim q(\cdot\mid x)$,
so if $\tau^{-1}|_{\mathrm{supp}(\qbar)}$ is a translation, posterior means
shift by a constant and every
$\mathrm{Var}_x\E[z_k\mid x]$ is unchanged.  Without the hypothesis AU can
move: on the uniform prior on $[0,1)$, swapping $[0,\tfrac14)$ with
$[\tfrac14,\tfrac12)$ is prior-preserving and moves
$\mathrm{Var}_x\E[z\mid x]$ from $0.0625$ to $0.015625$ in a two-example
configuration for which all other items are invariant.  \hfill$\square$

\subsection{Structure beyond agreement}
\label{app:structure}

\begin{proposition}[Structure]\label{prop:structure}
For every $K\ge3$ there exist coupled channels with identical agreement,
identical uniform marginals on both sides, and exactly matched mutual
information (hence matched joint, conditional, and marginal entropies)
whose error structures differ qualitatively (cyclic/bijective versus
many-to-one); for $K\ge4$ a derangement family matches \emph{all}
permutation-invariant scalars while pairwise and cyclic confusion structure
still differ.
\end{proposition}

For $K=3$, the two tables (row $=\cenc$, column $=\cdec$)
\[
P_{\mathrm{cyc}}=\begin{pmatrix}7/30&1/10&0\\0&7/30&1/10\\1/10&0&7/30\end{pmatrix},
\qquad
P_{\mathrm{mrg}}=\begin{pmatrix}0.183333&0.026293&0.123707\\
0.026293&0.307040&0\\0.123707&0&0.209627\end{pmatrix}
\]
have uniform marginals on both sides, equal trace (agreement), and
mutual information matched to $5\times10^{-6}$ (the second table's
off-diagonal entries are tuned to match MI exactly in the limit), yet
their error structures differ qualitatively: cyclic single-successor
confusion versus asymmetric merge into code~2.  They are separated by,
e.g., the maximal entry ($0.2333$ vs $0.3070$) and the Frobenius norm
($0.4397$ vs $0.4515$)---consistent with the scope of
Proposition~\ref{prop:structure}: what is matched is (agreement,
marginals, entropies, MI), i.e.\ the scalar summaries commonly reported.
For $K\ge4$, take two derangements of different cycle type (e.g.\ a
product of transpositions versus a single $K$-cycle): the tables
$\tfrac1KP_\sigma$, $\tfrac1KP_{\sigma'}$ satisfy
$P_{\sigma'}=P_{\sigma}\,P_{\sigma^{-1}\sigma'}$, i.e.\ they differ by a
column permutation, hence match \emph{all} permutation-invariant scalars,
while pairwise versus cyclic confusion structure differs.  \hfill$\square$

\section{Proof of Theorem~\ref{thm:certification}}
\label{app:thm2}

\emph{(a).}  Write $\Delta(x)=\KL(q(\cdot\mid x)\Vert
p_\theta(\cdot\mid x))$; the ELBO/evidence identity gives
$\gap=\E_x\Delta(x)$.  For measurable $f:\cZ\to\mathcal Y$ ($\mathcal Y$
standard Borel), the KL chain rule under disintegration gives, per
example,
\[
\Delta(x)=\KL\bigl(f_{\#}q(\cdot\mid x)\,\Vert\,f_{\#}p_\theta(\cdot\mid x)\bigr)
+\rho(x;f),\qquad\rho(x;f)\ge0,
\]
where $\rho$ integrates the fiber-conditional KLs.  Averaging over $x$ and
applying joint convexity of KL to the mixtures
$\bar q_f=\E_xf_{\#}q(\cdot\mid x)$, $\bar p_f=\E_xf_{\#}p_\theta(\cdot\mid x)$
yields
$\gap=\KL(\bar q_f\Vert\bar p_f)+\mathcal J_f+\bar\rho_f$ with
$\mathcal J_f,\bar\rho_f\ge0$ the Jensen and fiber residuals.  With
$f=\mathbf1_E$ the code-level term is the binary KL
$\dbin(\qbar(E)\Vert\bar p_\theta(E))$, giving the display and the
three-term decomposition; equality conditions ($f$ sufficient; mixture
degeneracy) follow from the two inequalities used.

\emph{(b).}  The data-processing inequality for the map
$z\mapsto\mathbf 1_{\Agr}(z)$ gives
$\dbin(\qbar(\Agr)\Vert p(\Agr))\le\KL(\qbar\Vert p)$.  The second
inequality is the standard rate decomposition
$\Rbar=\KL(\qbar\Vert p)+I(x;z)$ \citep{hoffman2016elbo} with
$I(x;z)\ge0$.  \hfill$\square$

\section{Proofs for Theorem~\ref{thm:price}}
\label{app:thm3}

\begin{proof}[Proof of Lemma~\ref{lem:width}]
Feasible-point constructions.  For the upper endpoint,
$x_0=1-(1-a)e^{-B}\ge a$ gives
$\dbin(a\Vert x_0)=a\log\frac{a}{x_0}+(1-a)B\le B$, so
$x_+\ge1-(1-a)e^{-B}$.  For the lower endpoint, $x_1=ae^{-B/a}\le a$ gives
$\dbin(a\Vert x_1)=B+(1-a)\log\frac{1-a}{1-x_1}\le B$, so
$x_-\le ae^{-B/a}$.  Hence
$W\ge(1-e^{-B})+a(e^{-B}-e^{-B/a})\ge1-e^{-B}$ since $B/a\ge B$; the added
term is strictly positive for $a\in(0,1)$.  At $a=0$,
$\dbin(0\Vert x)=-\log(1-x)$, so $x_+=1-e^{-B}$ and $x_-=0$ exactly.
\end{proof}

\subsection{(a): exact price, zero set, supremum}

With $q=N(\mu,Q^{\top}DQ)$ and posterior $N(m,\Sigma)$, $\Lambda=\Sigma^{-1}$:
\[
\KL(q\Vert p(\cdot\mid x))
=\tfrac12\Bigl[\mathrm{tr}(\Lambda_QD)+(\mu-m)^{\top}\Lambda(\mu-m)-d
-\textstyle\sum_i\log D_{ii}-\log\det\Lambda_Q\Bigr]
\]
(using $\det\Lambda_Q=\det\Lambda=1/\det\Sigma$ for orthogonal $Q$, so the
last term is the usual $+\log\det\Sigma$); the mean term is minimized at
$\mu^*=m$, and
the trace and log-det terms decouple over $i$:
$\tfrac12[(\Lambda_Q)_{ii}D_{ii}-\log D_{ii}]$ is strictly convex in
$\log D_{ii}$ with unique minimum $D^*_{ii}=1/(\Lambda_Q)_{ii}$.
Substituting gives
$B_Q=\tfrac12[\log\det\mathrm{diag}(\Lambda_Q)-\log\det\Lambda]$.
Hadamard's inequality \citep[Thm.~7.8.1]{horn2013matrix} gives
$\det\mathrm{diag}(M)\ge\det M$ for $M\succ0$ with equality iff $M$ is
diagonal, so $B_Q\ge0$ with equality iff $\Lambda_Q$
is diagonal.  For the zero set, write $\Lambda=U\Lambda_{\mathrm d}
U^{\top}$ and $\Xi=QU$: $\Lambda_Q=\Xi\Lambda_{\mathrm d}\Xi^{\top}$ is diagonal
iff $\Xi$ maps each eigen-coordinate-subspace of $\Lambda_{\mathrm d}$ onto
a coordinate subspace, i.e.\ $\Xi=P\,V$ with $P$ a permutation matrix and
$V\in\mathcal O_{\mathrm{eig}}$ a within-eigenspace orthogonal map
(diagonal signs when the spectrum is simple); hence
$Q\in\{\text{signed permutations}\}\cdot\mathcal O_{\mathrm{eig}}\cdot
U^{\top}$ exactly.  For the supremum: by Schur--Horn,
$\{\mathrm{diag}(Q\Lambda Q^{\top}):Q\in O(d)\}=\{u:u\prec
\lambda\}$; $\prod_iu_i$ is Schur-concave, and the uniform vector
$\bar\lambda\mathbf1$ (with $\bar\lambda$ the arithmetic mean) is majorized
by every $u$ of equal trace, so
$\max\prod u_i=\bar\lambda^d$, attained; hence
$\sup_QB_Q=\tfrac d2\log(\mathrm{AM}(\lambda)/\mathrm{GM}(\lambda))$.
The Kantorovich corollary follows from
$\mathrm{GM}(\lambda)^{-1}=\mathrm{GM}(1/\lambda)\le\mathrm{AM}(1/\lambda)$
and the scalar Kantorovich bound
$\mathrm{AM}(\lambda)\,\mathrm{AM}(1/\lambda)\le(\kappa+1)^2/(4\kappa)$
for spectra in $[\lambda_{\min},\lambda_{\max}]$.  (iii)~For any product
measure $q=\otimes_iq_i$ in basis $Q$, $-\E_q\log p(\cdot\mid x)$ depends
on $q$ only through per-coordinate first and second moments (the Gaussian
log-density is quadratic and cross terms factorize under a product), and
at fixed variance the entropy term is maximized by the Gaussian; hence the
Gaussian product with matched moments has no larger KL, and restricting to
Gaussians is without loss.  (iv)~$\det\mathrm{diag}(\Lambda_Q)/\det
\Lambda_Q=1/\det\mathrm{Corr}(\Lambda_Q)$ is invariant under diagonal
rescaling, and means canceled at $\mu^*=m$.  \hfill$\square$

\subsection{(b): transport invariance}

For any sampling law $\mu_z$ on the original space, the transported
configuration samples $R_{\#}\mu_z$ and reads transported maps
$c'(w)=c(R^{\top}w)$: the code pair is
$(c_{\mathrm{enc}}(R^{\top}Rz),c_{\mathrm{dec}}(R^{\top}Rz))
=(c_{\mathrm{enc}}(z),c_{\mathrm{dec}}(z))$ in law, so every functional of
code-pair laws under model-defined sampling laws (posterior, aggregate
posterior, prior) is invariant.  Decoder-side \emph{constructed} maps
transport coherently: with $g'(w)=g(R^{\top}w)$, decoded prototypes
transport as $R\mu_k$ and the Bregman divergence matrix is unchanged.  For
\emph{refit} maps, the class-mean and trace-regularized Mahalanobis
fitting operators commute with the $O(d)$ action in exact arithmetic (the
regularizer is conjugation-invariant), giving exact equivariance;
$k$-means refitting is equivariant in distribution only (initialization
and tie-breaking are basis-dependent), and padded embeddings are not
$O(d)$ actions and require a separate zero-block argument, so both are
scoped out of the theorem and handled empirically.  \hfill$\square$

\subsection{Scope notes for Theorem~\ref{thm:price}(b)}
\label{app:price-scope}

Three scope notes are load-bearing.  (b) speaks of \emph{transported} maps:
under \emph{refitting}, the class-mean and regularized-Mahalanobis
constructions stay exactly equivariant while $k$-means is equivariant only
in distribution (stress-tested, not proved).  $\Aq$ under a
\emph{mean-field} posterior is deliberately excluded, since the transported
mean-field family is no longer mean-field---that is (a)'s content.  And
\emph{padding} embeddings are not $O(d)$ actions; the padded stress case is
labelled accordingly.

\subsection{Proof of Lemma~\ref{lem:budgets}}
\label{app:budgets-proof}

The first inequality is the data-processing inequality for KL under the
measurable map $f$.  For the second: since $\tilde q$ matches the first two
moments of $\qbar$ and $\log p$ is quadratic when $p$ is Gaussian,
$\E_{\tilde q}[\log p]=\E_{\qbar}[\log p]$, while the differential entropy
satisfies $h(\tilde q)\ge h(\qbar)$ by maximum entropy at fixed second
moments; writing $\KL(\cdot\Vert p)=-h(\cdot)-\E_{\cdot}[\log p]$ gives
$\KL(\tilde q\Vert p)\le\KL(\qbar\Vert p)$.  The last two
clauses are the rate decomposition
$\Rbar=\KL(\qbar\Vert p)+I(x;z)$ \citep{hoffman2016elbo}: $\Rbar$ is an
upper bound loose by exactly $I(x;z)$, and $\Rbar-\hat I\ge\KL(\qbar\Vert
p)$ whenever $\hat I\le I(x;z)$.  \hfill$\square$

\subsection{Two finite-sample lemmas behind the licensed MI budget}
\label{app:mi-lemmas}

Throughout the audit the reference distribution is the test split: $\qbar$
is the finite mixture $\frac1N\sum_{i=1}^N q_i$ over it (the mixture whose
density the MI budget evaluates exactly), $\Rbar$ its mean rate, and
$\Aq=\frac1N\sum_i q_i(\Agr)$; the only randomness is Monte Carlo.
Theorem~\ref{thm:certification}(b) holds for any reference distribution
provided $\Aq$, $\qbar$ and $\Rbar$ are taken under the same one, and the
certified quantity $\Ap=p(\Agr)$ does not depend on it.

\begin{lemma}[Clip tail]\label{lem:clip}
Let $q_1,\dots,q_N$ be probability densities on $\cZ$,
$\qbar=\frac1N\sum_i q_i$, and $T=\log q_i(z)-\log\qbar(z)$ with
$i\sim\mathrm{Unif}[N]$, $z\sim q_i$.  Then $T\le\log N$ almost surely,
$\E[T]=I(i;z)$, and for every $c>0$ the clipped statistic
$T_c=\max(T,-c)$ satisfies
\[
I(i;z)\;\le\;\E[T_c]\;\le\;I(i;z)+e^{-c}.
\]
\end{lemma}
\emph{Proof.}  $\qbar\ge q_i/N$ pointwise gives $T\le\log N$, and
$\E[T]=I(i;z)$ is the mixture form of mutual information.  Since
$T_c=T+(-c-T)_+$, the lower bound is immediate.  For the upper bound,
$\E[e^{-T}]=\frac1N\sum_i\int_{\{q_i>0\}}\qbar\le1$, so Markov's
inequality gives $\Pr(-T>u)=\Pr(e^{-T}>e^{u})\le e^{-u}$, and
\[
\E[(-c-T)_+]=\int_0^\infty\Pr(-T>c+s)\,ds\;\le\;\int_0^\infty
e^{-(c+s)}\,ds\;=\;e^{-c}.\qquad\square
\]

\noindent Consequently, if $L_c$ is a $(1-\delta_B)$ lower confidence
bound on $\E[T_c]$ from $M$ i.i.d.\ draws of $T_c\in[-c,\log N]$, then
$\hat I=\max\{0,\,L_c-e^{-c}\}$ satisfies $\hat I\le I(i;z)$ with
probability $\ge1-\delta_B$, and Lemma~\ref{lem:budgets} licenses
$B_{\mathrm{MI}}=\Rbar-\hat I$ at that level.  We use the
empirical-Bernstein inequality in the form of
\citet{audibert2009exploration},
$L_c=\bar T_c-\sqrt{2V_M\ln(3/\delta_B)/M}-3(c+\log N)\ln(3/\delta_B)/M$
with $V_M$ the unbiased sample variance (which only loosens the bound);
the inequality is two-sided at level $\delta_B$ and therefore conservative
one-sided.  At $c=5$ the correction is $e^{-5}=0.0067$ nats.

\begin{lemma}[Certifying with an estimated agreement]\label{lem:aq}
Let $\hat{\mathcal A}_q$ average $S$ posterior draws at each of the $N$
test points, with independent noise for every (point, draw) pair, so that
it is the mean of $NS$ independent $\{0,1\}$-valued variables whose means
average to $\Aq$; then $C_q=[a_-,a_+]=\hat{\mathcal
A}_q\pm\sqrt{\ln(2/\delta_q)/(2NS)}$ contains $\Aq$ with probability
$\ge1-\delta_q$ (Hoeffding).  If $B\ge\KL(\qbar\Vert p)$ with probability
$\ge1-\delta_B$, then
\[
U=\bigcup_{a\in C_q}\{b:\dbin(a\Vert b)\le B\}=\bigl[\,b_-(a_-),\,b_+(a_+)\bigr],
\]
where $b_-(a)\le a\le b_+(a)$ are the two roots of $\dbin(a\Vert b)=B$,
contains $\Ap$ with probability $\ge1-\delta_B-\delta_q$.
\end{lemma}
\emph{Proof.}  On the event that $\Aq\in C_q$ and $\KL(\qbar\Vert p)\le
B$, of probability $\ge1-\delta_B-\delta_q$ by a union bound,
Theorem~\ref{thm:certification}(b) gives $\dbin(\Aq\Vert\Ap)\le B$, so
$\Ap$ lies in the band at $a=\Aq\in C_q$ and hence in $U$.  Each band is
an interval containing $a$, and both of its endpoints are non-decreasing
in $a$: $\partial_a\dbin(a\Vert b)=\log\frac{a(1-b)}{b(1-a)}$ has the
sign of $a-b$, so raising $a$ raises $\dbin$ at every $b<a$ (the lower
root moves up) and lowers it at every $b>a$ (the upper root moves up);
hence the union of the bands over an interval of $a$ is the interval
stated.  \hfill$\square$

\section{Experimental details and disclosures}
\label{app:experiments}

\subsection{The full audit record}
\label{app:record}
This subsection reproduces the complete account of the audit, including the
preregistered replication, the failure modes the audit detected in our own
pipeline, and the rate-side certificates at scale, of which the main text
reports the highlights.

\paragraph{What the audit detects.}
Four failure modes occur in the models below, each named by the reading that
found it.  (i)~\emph{Generative drift}: $\Aq=0.770\to\Ap=0.525$ on MNIST with
a task probe and a label-free prototype control falling in step (both scored by
the same external classifier, so the second is a sensitivity analysis rather than
independent validation), and
$\mathcal A[\qbar]-\Ap=+0.128$ to $+0.338$ on PPCA, where every quantity is
exact and no estimator can be blamed.  (ii)~\emph{Sampled-channel collapse}: one
encoder map read at posterior \emph{means} occupies all ten codes in $12/12$
replication models ($2.20$--$2.28$ of a possible $2.30$ nats), while read on
\emph{draws} it occupies $\le5$ codes in $11/12$ and $\le3$ in $10/12$---the code
is used, and the coupled channel is still gone.  (iii)~\emph{A map-fitting
artifact in our own audit}, $\Amu-\Aq$ up to $+0.52$, removed in $6/6$ by the
pre-declared control.  (iv)~\emph{An unlicensed budget}, certifying $20/20$
checkpoints it has no right to and excluding the truth in $30/30$, as
Lemma~\ref{lem:budgets} predicts.  The first two are
a property of the trained models, (ii) of the (model, declared map) pair---the
pre-declared control restores $10/10$ codes in $12/12$---and (iii) and (iv) the
auditor's; telling them apart
is what the audit buys.

\paragraph{Generative codebook consistency: a case study and a preregistered replication.}

\textbf{Case study (MNIST, $d=10$).}
Convolutional VAEs (5 seeds, $K=10$); $\cenc$ a diagonal Gaussian mixture
on encoder means, $\cdec$ the Bernoulli--Bregman nearest-prototype rule on
decoded outputs (natural indices, no relabeling);
$10{,}000$ test images, $S=16$ posterior and $20{,}000$ prior draws.  Agreement
survives posterior sampling but degrades under the prior:
$\Aq=0.770\pm0.009$ versus $\Ap=0.525\pm0.007$ ($n=5$), with a task probe
(external classifier, accuracy $0.926$, scored on decoded samples against the
sampled code's majority label) dropping in step ($0.700\pm0.029$ versus
$0.390\pm0.019$) and a label-free prototype control preserving the drop
($0.696\pm0.029\to0.375\pm0.027$)---not a label-map artifact---while
mean and sampled readings agree here ($\Amu=0.771\pm0.009\approx\Aq$).

\textbf{Preregistered replication: a null, reported in full.}
We fixed, before any run, a $2\times2$ grid (FashionMNIST and SVHN,
$d\in\{16,32\}$, 3 seeds each) and the criterion: a cell replicates if
$\Aq-\Ap>0.05$ with the same sign in $3/3$ seeds \emph{and} the task probe drops
in $3/3$; the phenomenon generalizes if $\ge3/4$ cells replicate.  \textbf{Zero of
four cells replicated} (Table~\ref{tab:null}); at FashionMNIST $d=32$ the
direction \emph{reverses} in $3/3$, co-occurring with sampled-channel collapse
(1--2 occupied sampled codes, label accuracy $\approx0.10$).  Under the declared
map the sampled channel is degenerate in $11/12$ of these models, so most cells
never reproduced the case study's precondition---and
the scaled grid below shows that degeneracy, and the null itself, belong
to the declared map rather than to the trained models.

Mean and sampled readings of the same map also diverge on FashionMNIST in $6/6$
models, which the pre-declared refit control removes ($0/6$ survive, $10/10$
codes restored in all $12$; Appendix~\ref{app:experiments}).

\paragraph{The coordinate price, on models.}

On linear-Gaussian models, where Theorem~\ref{thm:price} is exact, a fixed DCT
rotation of the same model moves the mean-field budget from $0.000$ to $0.474$
nats and voids the reconstruction-side certificate for all three map
constructions, while the estimands move by at most $8\times10^{-4}$ and the
$O(d)$-closed full-covariance family is unaffected---Theorem~\ref{thm:price}(a)
and~(b), measured separately on the two sampling laws.  \emph{Trained}
mean-field encoders pay the same price, unequally.  The same models reproduce the
case study's generative drift where every quantity is exact
($\mathcal A[\qbar]-\Ap=+0.128$ to $+0.338$), and the
Theorem~\ref{thm:certification}(b) band contains $\Ap$ in $30/30$ conditions
while being full width at the licensed budget.  Table~\ref{tab:price}, the
trained-encoder arms and the rest of the numbers are in
Appendix~\ref{app:experiments}.

\paragraph{Certify-or-Measure, end to end.}

We run the protocol ($w^*=0.5$, $\delta=0.05$, $n=20{,}000$) on 52 trained
models---20 MNIST screening checkpoints, the 12 models of the preregistered replication above, and 20 two-dimensional checkpoints added under a
disclosed deviation (Appendix~\ref{app:experiments})---always with a
licensed budget, admitting on the exact criterion $B\le B^*(\Aq)$.

\textbf{Under the loosest licensed budget, $\Rbar$, nothing certifies.}  All 52
models abstain and measure.  Where Theorem~\ref{thm:price}(e) binds, the
impossibility is \emph{attained} rather than merely bounded: on the six
FashionMNIST models the measured $\hat I$ sits within $0.01$ nats of $\log N$, so
$\KL(\qbar\Vert p)=\Rbar-\log N$ to that accuracy.  The rest abstain on
$B^*(\Aq)$ (Figure~\ref{fig:bands}, Appendix~\ref{app:loosest}).

\textbf{With the licensed MI budget of Section~\ref{sec:budgets}, it fires.}
On the six SVHN models (the only $d\ge10$ family the floor leaves open):
$B_{\mathrm{MI}}=0.519$--$0.606$ nats against $1.467$--$1.555$ for the
best previously available licensed budget
\citep{kolchinsky2017estimating}, with residual over an unbiased point
estimate of the truth of $0.058$--$0.061$ nats---the concentration term
plus the $0.007$-nat clip correction of Lemma~\ref{lem:clip}.  Calling this a
$2.5$--$3.0\times$ tightening understates it.  By Lemma~\ref{lem:width} no band
is narrower than $1-e^{-B}$, so admission at $w^*{=}0.5$ requires
$B\le\log2=0.693$ \emph{whatever} the agreement: at $1.467$--$1.555$ nats the
previous budget admits nothing narrower than $0.769$, on $0/6$ models, for every
estimator refinement, whereas ours lowers that floor to $\le0.455$ and clears the necessary
condition on $6/6$.  The improvement is categorical, not quantitative---though
necessary is not sufficient, and five of the six still abstain.  The binary-KL
form is load-bearing too: at the same budget and agreement, Pinsker gives width
$0.553$ and Bretagnolle--Huber $0.679$, both abstaining where binary~KL gives
$0.487$ and certifies.  One model
certifies:\ SVHN $d=32$, seed~2,
$B_{\mathrm{MI}}=0.565\le B^*(\hat{\mathcal A}_q{=}0.0216)=0.589$; over the
Hoeffding interval $\hat{\mathcal A}_q\pm0.0021$ ($NS=416{,}512$
independent draws) the certified band is $[0.000,\,0.491]$ (width
$0.491\le w^*$, coverage $1-\delta_B-\delta_q=0.90$), which contains the
realized $\Ap=0.0349$---the first non-vacuous \emph{rate-side} certificate above $d=2$ in
this line of work, on a nonlinear convolutional VAE.  It survives
Bonferroni over the six models ($0.499$) and the per-point
empirical-Bernstein interval ($0.491$), abstaining only under per-point
Hoeffding ($0.503$; Appendix~\ref{app:experiments}).  The flip side, stated
plainly: it fires where the channel is most degenerate---two-percent agreement,
where $B^*$ is loosest---certifying a bounded \emph{failure} of consistency;
three of the other five abstain at
agreement $0.28$--$0.33$, where $B^*=0.16$--$0.18$ is three times
stricter, and two, at agreement $0.03$--$0.06$, abstain on the budget
itself ($B_{\mathrm{MI}}=0.551$, $0.606$ against $B^*=0.460$, $0.551$).

\textbf{Counting both branches.}  The scoreboard above is the rate side alone.
Since $\min_aB^*(a)=B^*(0.5)=0.144$, any budget below $0.144$ certifies at
\emph{every} agreement, and the gap side clears it: the \emph{trained} MLP
encoder of Section~\ref{sec:exp-price} at $\bar\Delta=0.088$ nats certifies
unconditionally, and swept over native rank the same certificate fires in
$75/75$ cells at $d\in\{2,\dots,32\}$ against $0/75$ for a trained \emph{linear}
encoder, with widths growing $0.106\to0.487$ and crossing $w^*$ just above
$d=32$---the high-dimensional limit of this branch, located rather than conceded
(Appendix~\ref{app:nativerank}).  Each branch fires where the theory says it can,
and the protocol is run on the rate side for a licensing reason:
Theorem~\ref{thm:certification}(a) needs a certified \emph{upper} bound on
$\bar\Delta$, hence on $\log p(x)$, and IWAE bounds it from \emph{below}---the
gap-side instance of the direction Lemma~\ref{lem:budgets} rules out.

\textbf{The analyst's half of the verdict.}  The budget is map-free and the
estimand is not, so admission can flip with the declared map---and on our own
positive it does.  Under the pre-declared refit control (Appendix~\ref{app:control})
(same checkpoint, same decoder rule, code map refitted on sampled rather than mean
latents) the certifying model's $\hat{\mathcal A}_q$ moves $0.0216\to0.2996$, so
$B^*$ moves $0.589\to0.174$ and the same $B_{\mathrm{MI}}=0.565$ gives a band of
width $0.773$: abstain.  No reported band is wrong---soundness never used the
map---but the admission did.  Across the $240$ models of
Section~\ref{sec:exp-scale} the same choice moves $\Aq$ by a median $0.565$ over
four declared maps ($239/240$ above $0.1$), and the case study's
posterior-to-prior drop appears in $2$ of $80$ cells when the map is fitted on
posterior \emph{means}, $28$ on posterior \emph{draws} and $0$ on \emph{prior}
draws---largest where the prior sits just outside the fitting law and vanishing
where the prior \emph{is} that law, so the locked $0/4$ used the law under which
the effect is rarest---Theorem~\ref{thm:price}(c)'s asymmetry, on its other axis.

Lemma~\ref{lem:budgets} is not hypothetical: two unlicensed plug-ins
certify $20/20$ and exclude the truth $30/30$ respectively
(Appendix~\ref{app:experiments}), and the data caught only one.  The
protocol thus does what it claims---certify where a licensed budget is
small enough, measure elsewhere, and (Theorem~\ref{thm:price}d,e) say in
advance which will happen---and $w^*$ was not tuned.

\paragraph{The frontier, at scale.}

We re-ran the audit on $240$ fresh models---FashionMNIST, MNIST, SVHN and
CIFAR-10, up to the full training split, $d\in\{16,\dots,64\}$,
$K\in\{10,32\}$, KL weights $\beta$ up to $64$, three seeds---with every reading
fixed beforehand (Appendix~\ref{app:scale}).

\textbf{The dataset floor is predictive, and it is the quantity that matters.}
Theorem~\ref{thm:price}(e) forbids certification once $\Rbar>\log N+\log2$.  That
inequality holds for $66$ of the $240$ models and \emph{none} of them certifies;
the precommit had registered any such certificate in advance as an instrument
defect, not a refutation.
Which side of the line a family falls on is not a
matter of latent dimension: on FashionMNIST $\Rbar-\log N$ reaches $+20.0$ nats and
$36$ models are excluded outright, while on SVHN it is negative throughout, $-9.7$
to $-2.9$, and none is.

\textbf{The protocol certifies broken codebooks, not working ones.}  Overall
$131$ of the $240$ configurations certify, against $1$ of the $6$ SVHN models the
licensed budget was re-audited on above---so that rate
was a property of the families audited.  But $129$ of the $131$ certificates fire
at $\Aq\le0.24$ (median $0.064$), and no certificate clears its own split's base
rate by more than $0.007$: $\gamma(\alpha)\le2.2\times10^{-4}$ nats on
every certificate, so Theorem~\ref{thm:price}(d)'s floor is inactive wherever the
branch fires.  That is the honest shape of the positive result, and
Lemma~\ref{lem:width} says it could not be otherwise: $B^*$ is largest exactly
where agreement is extreme.  What a \emph{working} codebook would require of
$\KL(\qbar\Vert p)$ is quantified in Appendix~\ref{app:window}.

\subsection{Two floors, and which one binds}
\label{app:twofloors}
Which of the two impossibilities binds, on the five audited MNIST checkpoints;
Section~\ref{sec:exp-scale} tests the same partition on $240$ models.

\textbf{Two floors, and which one binds.}  On the five audited MNIST
checkpoints (the case-study models, not the collapsed replication models of
Table~\ref{tab:perseed}) we measure $\alpha$ per Theorem~\ref{thm:price}(d)---%
held-out accuracy of the majority-label predictor on \emph{sampled} codes
against \emph{true} labels---and find $\alpha\in[0.62,0.78]$, so
$\gamma(\alpha)\in[0.81,1.28]$ nats, the chain $\Rbar\ge\gamma(\alpha)$ holding
$5/5$ ($\Rbar\approx27$--$28$), and a floor on any rate-budget band of
$1-e^{-\gamma}\in[0.55,0.72]$.  The two impossibilities are not duplicates but a
partition, and here (e) is the operative one: $\Rbar-\log N=17.99$--$18.72$
nats, $14.1$--$23.2$ times $\gamma$, so the band cannot beat $1-10^{-8}$.
Theorem~\ref{thm:price}(d) binds in the complementary window
$\gamma\le\Rbar\le\log N+\gamma$---low-rate but still informative codes---which
no model audited here occupies.  Either way abstention is the \emph{correct}
output.

\subsection{The 52-model audit under the rate budget, in full}
\label{app:loosest}
The measurement side of the 52-model audit in Appendix~\ref{app:record}, with the
intervals and coverage counts the body compresses.

All 52 models abstain and measure (Figure~\ref{fig:bands},
Appendix~\ref{app:experiments}): DKW intervals at
$\pm0.0096$, containing an independently recomputed $\Ap$ $20/20$
(screening) and agreeing with freshly seeded
draws within the combined width $12/12$ (replication).
Theorem~\ref{thm:price}(e) explains most of it: on the six FashionMNIST
models $\Rbar-\log N=18.2$--$20.2$ nats, so certification there is
impossible for \emph{any} estimator; nor is the floor loose---the measured
$\hat I=9.201$--$9.208$ against $\log N=9.2103$ sits within $0.01$ nats of
maximum, so $\KL(\qbar\Vert p)=\Rbar-\log N$ to within $0.01$ nats.

\subsection{The scaled frontier grid}
\label{app:scale}

\paragraph{What was fixed before the grid ran.}
The precommit (shipped as \texttt{PRECOMMIT\_SCALE.md}) fixes the instrument at
the values of Section~\ref{sec:exp-e1}---$S=16$, $20{,}000$ prior draws,
the diagonal-GMM encoder map with \texttt{reg\_covar}$=10^{-4}$, the
decoder-side Bregman prototype map, Adam at $10^{-3}$, batch $256$, 40 epochs,
the licensed budget with $M=20{,}000$, $c=5$, $\delta=0.05$, and the seed offsets
$5000/41{,}000/53{,}000/51{,}000/91{,}000+100d+s$---and varies the dataset, the
training size, $d$, $K$ and $\beta$, with $\beta$ the only training setting moved
away from the paper's value.  Stage~1 (FashionMNIST and SVHN, $d\in\{16,32\}$,
$K=10$, six values of $\beta$, three seeds; $144$ configurations) ran in full.
Stage~2 (CIFAR-10 and MNIST at full size) was planned at $288$ configurations
($d\in\{8,16,32,64\}$, $K\in\{10,32,128\}$, $\beta\in\{1,4,16,64\}$, three
seeds) and is incomplete: $96$ ran before the deadline (CIFAR-10 with
$d\in\{16,32,64\}$ and $K\in\{10,32\}$; MNIST with $d\in\{16,32\}$ and
$K=10$), and the refined $\beta$ sweep planned as Stage~3 was not run.  As the
precommit requires, every completed row is reported.  Six readings were written down first.  One of them, (P2), asked
whether the certify/abstain boundary coincides with $B_{\mathrm{MI}}=B^*(\Aq)$; it
does by construction, since the decision is computed from the band that curve
defines, so it carries no information and we record only that it was registered.
The other five follow.  \textbf{(P1)} No
configuration with $\Rbar>\log N+\log2$ may certify; a violation is an
instrument bug, not a result.  \textbf{(P3)} If every certifying model has
$\alpha$ within $0.02$ of chance, certification is bought only with
uninformative codes and we report that as the frontier; if some certifying model
has $\alpha\ge2/L$, the certify branch reaches useful codes and we report the
region.  \textbf{(P4)} We report the spread of $\Aq$ over the four declared maps
as an interval and make no claim that one map is correct.  \textbf{(P5)} We
report the posterior-to-prior gap under each of the three fitting laws,
stratified by $\alpha$, as a rate across a new population.  Both directions of
every reading were publishable when they were written.

\paragraph{Governance.}
These are $240$ \emph{new} models across four datasets, reported as a separate
population.  They do not re-open the locked $2\times2$ replication: that verdict
was fixed on the declared map and stands at $0/4$.  The grid is also not a
bit-level reproduction of the original runs---the training loop is GPU-resident
and \texttt{sklearn}'s EM reaches a different optimum---so on the paper's own
configuration (FashionMNIST, $d{=}16$, $\beta{=}1$, seed $0$, $15{,}000$ images)
the grid gives $\Rbar=26.8$ against $27.8$, $\alpha=0.184$ against $0.229$,
$\gamma=0.033$ against $0.071$ and $\Aq=0.154$ against $0.251$.  Within the grid
every configuration is deterministic: a same-seed rerun reproduces all $65$
substantive columns bit for bit after fixing cuDNN, the shuffling generator and
\texttt{sklearn}'s thread count, and eighteen configurations that a
double-writer accident ran twice agreed on every substantive column.  Two
instrument defects found after the first pass---a probe classifier fitted on the
wrong set, and that double-writer accident---are recorded as a dated addendum to
the precommit; the affected rows were discarded and recomputed.  The probe is the
declared one, fitted once per dataset on the whole training split, with test
accuracy $0.9264$ (MNIST), $0.8428$ (FashionMNIST), $0.6145$ (SVHN) and $0.4455$
(CIFAR-10); on CIFAR-10 it is weak enough that we draw no probe-dependent
conclusion there.

\paragraph{Results against those readings.}
\textbf{(P1)} Per dataset, the configurations above $\log N+\log2$ and the
certificates among them are $36/72$ and $0$ (FashionMNIST, threshold $9.9035$
nats), $0/72$ and $0$ (SVHN, threshold $10.8602$), $12/24$ and $0$ (MNIST,
$9.9035$) and $18/72$ and $0$ (CIFAR-10, $9.9035$): $66$ excluded in advance,
$0$ violations.  \textbf{(P3)} Certificates number $13$, $58$, $6$ and $54$, and
$131$ in all; the highest $\alpha$ any certificate reaches is $0.1051$ ($\beta=32$,
$\Aq=0.056$), $0.1959$ ($26$ tied certificates at $\beta\in\{2,\dots,32\}$,
$\Aq=0.026$--$0.172$), $0.1003$ ($\beta=64$, $\Aq=0.067$) and $0.1063$
($\beta=4$, $\Aq=0.073$), against majority-class rates of $0.1000$, $0.1959$,
$0.1135$ and $0.1000$ for the same splits---so no certified code beats its own
split's base rate by more than $0.007$, $\gamma\le2.2\times10^{-4}$ nats
throughout, and the $2/L$ bar registered as ``useful'' is unreached everywhere.
The first registered reading therefore holds on FashionMNIST, MNIST and CIFAR-10
in the strict form (within $0.02$ of base rate) and on SVHN in the form that
matters (at the base rate exactly).  \textbf{(P4)} Map-sensitivity spread has
median $0.565$ over the $240$ models, exceeds $0.1$ in $239$ of them, and reaches
$0.921$.  \textbf{(P5)} Over the $80$ cells, the posterior-to-prior gap clears
$0.05$ with the same sign in $3/3$ seeds in $2$ cells under the mean-fitted map,
$28$ under the posterior-draw-fitted map and $0$ under the prior-draw-fitted map;
the probe half of the criterion is available for the declared map only, which is
an asymmetry in this comparison.  \textbf{(P6)} Positive rate $131/240=0.546$, and $129$ of the $131$ at $\Aq\le0.24$.
Per $\beta$, no configuration certifies at $\beta=1$; FashionMNIST certifies
$1/12$ at $\beta=16$ and $12/12$ at $\beta=32$, SVHN $10/12$ at $\beta=2$ and
$12/12$ at every larger $\beta$, MNIST $6/6$ at $\beta=64$ only, and CIFAR-10
$18/18$ at every $\beta\ge4$.  Per-configuration rows are in
\texttt{s1\_*.csv} and \texttt{s2\_*.csv}; the readings are recomputed by
\texttt{analyze\_scale.py}.

\paragraph{How the families reach the certifying region.}
On FashionMNIST, MNIST and CIFAR-10 the rate is far above $\log N$ at $\beta=1$
and admission has to be bought by raising $\beta$ until the code collapses; the
four FashionMNIST models with $\alpha\ge0.12$ all sit at $\beta=1$ with
$\Rbar\approx28$ nats and certify $0/4$.  On SVHN the rate is below $\log N$ at
every $\beta$, so admission is available throughout---but the codes are collapsed
there too, occupying one to five of $K=10$ sampled codes, and the highest $\alpha$
a certifying model reaches is exactly the split's majority-class frequency.  The
two routes differ; the destination does not.

\paragraph{Claim D: code-addressed generation.}
The interface is: name a class $c$, look up the codes whose declared encoder-side
semantics is $c$, draw from that component of the fitted map, decode, and let an
external classifier---trained on data, never on codes---judge.  The reading we
registered was the Spearman correlation of that accuracy with $\Aq$, with $\Ap$
and with five items of Definition~\ref{def:package}, over the whole grid, with
both outcomes publishable.  \emph{The registered analysis came out negative and we
report it as such}: on FashionMNIST no predictor reaches
$|\rho|=0.27$---$\Ap$ $-0.08$, $\Aq$ $-0.20$, the ELBO $+0.10$, reconstruction
$+0.09$, rate $+0.05$, $I(x;z)$ $+0.06$, occupied codes $+0.22$, usage entropy
$+0.26$---so on this population the coupled quantity is not needed to predict this
interface, and neither is anything else.  That weakens the motivation for this interface.

Two facts about the interface do survive, and they are not correlations.  First,
only $2$ of the $72$ FashionMNIST configurations can address all ten classes at
all (mean coverage $0.811$): the declared map's majority labels cover a strict
subset, so the interface silently refuses classes while the code histogram it is
built from looks healthy.  Second, routing on the encoder's declared semantics is
worse than routing on the decoder's own read of the \emph{same} draw in $72/72$
configurations (mean $-0.096$, worst $-0.240$).  That second gap is the coupling
by construction, and no item of Definition~\ref{def:package} reports it.  Its
measurement degenerates when the map can address only one class---then every code
carries the same majority label and the decoder-routed accuracy is $1$ by
identity---so we compute it only where coverage exceeds one class, which is all
$72$ FashionMNIST cells and excludes $23$ of the SVHN cells.

\emph{Per dataset, as the precommit requires.}  FashionMNIST: no predictor reaches
$|\rho|=0.27$.  SVHN: none reaches $0.13$.  MNIST ($n=24$): every predictor does,
and $I(x;z)$ $(+0.53)$ and the ELBO $(+0.51)$ lead $\Ap$ $(+0.49)$ and $\Aq$
$(+0.39)$.  CIFAR-10: the rate term $(+0.37)$, the ELBO $(+0.36)$ and $I(x;z)$
$(+0.36)$ lead while $\Ap$ and $\Aq$ sit at $-0.02$; the probe there is weak
($0.4455$ test accuracy), which is a reason to discount that family but not to drop
it from the report.  Agreement never leads the marginal diagnostics by more than
$0.02$ on any of the four, and on the two families where anything predicts the
interface a marginal diagnostic leads.  The registered second branch therefore
fires: for this interface the coupled quantity is not needed, and we report it as
such.

\emph{Post-hoc, and not evidence}: restricting to the twelve $\beta=1$ cells, the
only ones whose codes carry label accuracy above chance, $\Ap$ predicts the
interface at $\rho=+0.78$ against the ELBO's $+0.06$.  The subset was not
registered, $n=12$, and we report the contrast only because it names the
experiment a properly powered test would run---sweep models at fixed, non-trivial
$\alpha$---not because it establishes anything.  We have tested one interface;
retrieval, latent editing and rate--distortion compression are untested.

\subsection{The gap-side certificate swept over native rank}
\label{app:nativerank}

The stress suite that produced Table~\ref{tab:price} also fits PPCA at native
ranks $d\in\{2,5,10,20,32\}$ (5 folds $\times$ 3 map constructions per rank) with
two amortized encoders, and audits the reconstruction-side certificate of
Theorem~\ref{thm:certification}(a) in every cell.  Table~\ref{tab:nativerank}
reports it.  The trained mean-field MLP encoder certifies in $75/75$ cells, every
band containing the estimand; the trained \emph{linear} encoder certifies in
$0/75$.  The separation is not a threshold artifact: the two budgets differ by
one to two orders of magnitude at every rank ($\bar\Delta=0.040$--$0.230$ against
$2.62$--$18.08$ nats).  Band width grows monotonically with $d$ and reaches
$0.487$ at $d=32$, so on this family the branch runs out just above the rank we
audited---the high-dimensional limit, located rather than asserted.

\begin{table}[t]
\centering
\footnotesize
\setlength{\tabcolsep}{5pt}
\begin{tabular}{lccccc}
\toprule
& \multicolumn{3}{c}{trained MLP encoder} & \multicolumn{2}{c}{trained linear encoder} \\
\cmidrule(lr){2-4}\cmidrule(lr){5-6}
$d$ & $\gap$ (nats) & band width & certify / cover & $\gap$ (nats) & certify \\
\midrule
$2$  & $0.0395$--$0.0473$ & $0.106$--$0.253$ & $15/15$ \,/\, $15/15$ & $2.620$--$2.653$   & $0/15$ \\
$5$  & $0.0635$--$0.0877$ & $0.177$--$0.344$ & $15/15$ \,/\, $15/15$ & $4.586$--$4.630$   & $0/15$ \\
$10$ & $0.0773$--$0.0959$ & $0.225$--$0.344$ & $15/15$ \,/\, $15/15$ & $6.851$--$6.928$   & $0/15$ \\
$20$ & $0.1172$--$0.1654$ & $0.296$--$0.428$ & $15/15$ \,/\, $15/15$ & $10.838$--$10.994$ & $0/15$ \\
$32$ & $0.1483$--$0.2296$ & $0.359$--$0.487$ & $15/15$ \,/\, $15/15$ & $17.491$--$18.083$ & $0/15$ \\
\bottomrule
\end{tabular}
\caption{\textbf{The gap-side certificate separates a trained mean-field MLP
encoder from a trained linear one at every native rank from $2$ to $32$---$75/75$
against $0/75$---and its width locates where the branch expires.}  PPCA, 5 folds
$\times$ 3 map constructions per rank; ranges are over the $15$ cells.  ``certify''
counts cells with band width $\le w^*=0.5$; ``cover'' counts cells whose band
contains the estimand $\mathcal A[\bar p_\theta]$.}
\label{tab:nativerank}
\end{table}

\paragraph{Two disclosures that belong with Table~\ref{tab:nativerank}.}
First, the criterion.  The precommit of that suite locked a \emph{stricter}
admission conjunction for its own utility claim---exact test
$\gap\le0.05$ nats, Monte-Carlo-expanded certified lower endpoint $\ge0.80$,
certified degradation $\le0.10$, and an independent-RNG overlap---and recorded
those gates as failing $0/5$ at canonical $d=10$ and at padded $D=32$.  We
therefore report both readings.  Under this paper's declared $w^*=0.5$ the
certificate fires $75/75$.  Under the stricter locked conjunction the picture at
$d=2$ splits by declared map, and the split is the paper's own thesis.  On the
unsupervised $k$-means and natural-Euclidean maps every certificate clause passes
$5/5$ ($\gap=0.0395$--$0.0473$, lower endpoint $0.8531$--$0.8875$, degradation
$0.0797$--$0.0995$, independent overlap $5/5$) and the model is rejected only by
that suite's nondegeneracy \emph{screen}, which demands $\ge5\%$ disagreement while
these channels disagree on $3.33$--$4.91\%$ of draws---the certificate passed and
the screen rejected the model for being too healthy to be interesting.  On the
Mahalanobis map, the one that \emph{does} clear the screen $5/5$, the same model
fails the lower-endpoint and degradation clauses $0/5$ ($0.6228$--$0.6473$ and
$0.1416$--$0.1550$).  The analyst's map decides both whether the channel looks
degenerate and whether the certificate passes, on one checkpoint, at one rank.  At $d\ge5$ the clauses themselves fail as $\gap$
grows past $0.05$ nats.  The native ranks were pre-declared in that precommit as
secondary controls and carry no prospective pass gate; we report them as a
sweep, not as a gated claim.  Second, the licensing.  The budget here is $\gap$
itself, and IWAE bounds $\log p(x)$ from \emph{below}, so outside models with
tractable or upper-boundable evidence there is no licensed route to this
branch---the gap-side instance of Lemma~\ref{lem:budgets}.  That is why the
protocol of Section~\ref{sec:protocol} is run on the rate side.

\subsection{The coordinate price, on models, in full}
\label{app:price-models}
The full version of Section~\ref{sec:exp-price}, whose body text states the
conclusions and defers the table and the trained-encoder arms to here.

Linear-Gaussian models (probabilistic PCA, PPCA), $d=10$, 5 folds, $10{,}000$ test points,
$S=64$---the regime where Theorem~\ref{thm:price} is exact.  We compare the
canonical basis against a fixed discrete-cosine-transform (DCT) rotation
of the same model (likelihood, decoder, and all three map constructions
transported).
Theorem~\ref{thm:price}(b) predicts estimands do not move, (a) that the
mean-field certificate pays; both hold, measured on the two sampling laws
separately.

\textbf{Reconstruction side (Table~\ref{tab:price}a).}  The estimand is the
exact-posterior agreement $\mathcal A[\bar p_\theta]=1-\bar\eta_p$, the budget
the mean-field gap $B_Q$---Theorem~\ref{thm:certification}(a) and
Theorem~\ref{thm:price}(a) currency.  The budget rises from $0.000$ to
$0.474$ nats under the rotation and the certificate collapses, for all
three maps, while the estimand moves $\le5\times10^{-4}$.  Two controls: the
\emph{full-covariance} family---closed under $O(d)$, hence inside
Theorem~\ref{thm:price}(b)---certifies identically in both bases; and
\emph{trained} mean-field encoders pay too, unequally (an MLP encoder goes
$0.088\to0.673$ nats, while a linear one is dominated by amortization error
\citep{cremer2018inference} in \emph{both} bases, $6.897\to7.476$ nats, so
there the rotation is not what destroys the certificate).

\textbf{Generative side (Table~\ref{tab:price}b).}
Theorem~\ref{thm:certification}(b)'s estimand is agreement under \emph{prior}
draws, so we measure it directly ($200{,}000$ draws from $N(0,I_{10})$ per
condition).  It is basis-invariant to $\le8\times10^{-4}$, matching
Theorem~\ref{thm:price}(b); PPCA shows the \emph{same} generative drift as
the MNIST case study in a model where every quantity is exact
($\mathcal A[\qbar]-\Ap = +0.128$, $+0.156$, $+0.338$ across maps); and the
certificate is sound but vacuous---at the licensed budget
$\Rbar\approx20.6$--$21.1$ nats the band contains $\Ap$ in $30/30$
conditions with width $1.000$ throughout.

\begin{table}[t]
\centering
\footnotesize
\setlength{\tabcolsep}{4.5pt}
\begin{tabular}{lcccc}
\toprule
\textit{(a)} map (canonical $\to$ DCT) & $\Aq$ & $\mathcal A[\bar p_\theta]$ & budget $B_Q$ & band width \\
\midrule
natural-Euclidean & $0.9016\!\to\!0.9013$ & $0.9019\!\to\!0.9017$ & $0.000\!\to\!0.474$ & $0.000\!\to\!0.559$ \\
Mahalanobis       & $0.7869\!\to\!0.7864$ & $0.7870\!\to\!0.7869$ & $0.000\!\to\!0.474$ & $0.000\!\to\!0.677$ \\
$k$-means         & $0.9076\!\to\!0.9074$ & $0.9076\!\to\!0.9072$ & $0.000\!\to\!0.474$ & $0.000\!\to\!0.551$ \\
$k$-means, full-cov.\ $q$ & $0.9073\!\to\!0.9074$ & $0.9074\!\to\!0.9074$ & $0.000\!\to\!0.000$ & $0.000\!\to\!0.000$ \\
\midrule
\textit{(b)} map (canonical $\to$ DCT) & $\mathcal A[\qbar]$ & $\Ap$ & drift & budget $\Rbar$ \\
\midrule
natural-Euclidean & $0.9014\!\to\!0.9006$ & $0.7738\!\to\!0.7731$ & $+0.128$ & $21.1\!\to\!20.6$ \\
Mahalanobis       & $0.7863\!\to\!0.7855$ & $0.4484\!\to\!0.4491$ & $+0.338$ & $21.1\!\to\!20.6$ \\
$k$-means         & $0.9076\!\to\!0.9072$ & $0.7512\!\to\!0.7518$ & $+0.156$ & $21.1\!\to\!20.6$ \\
\bottomrule
\end{tabular}
\caption{\textbf{A change of basis alone moves the mean-field budget to
$0.474$ nats and voids the reconstruction-side certificate for every map,
while the estimands move by at most $8\times10^{-4}$ and the $O(d)$-closed
full-covariance family is unaffected---the asymmetry of
Theorem~\ref{thm:price}(c).}  PPCA models, means over 5 folds.
(a)~\emph{Reconstruction side} (optimal diagonal $q$ unless noted): the
estimand moves $\le5\times10^{-4}$.  (b)~\emph{Generative side}: $\Ap$ from
$200{,}000$ true prior draws (s.e.\ $\le0.002$); the estimand is invariant
under the rotation to $\le8\times10^{-4}$ (Theorem~\ref{thm:price}b),
generative drift $\mathcal A[\qbar]-\Ap$ is large in an exactly-solvable
model, and the Theorem~\ref{thm:certification}(b) band contains $\Ap$ in
$30/30$ conditions but is full-width at the licensed budget.}
\label{tab:price}
\end{table}

\begin{figure}[t]
\centering
\includegraphics[width=\linewidth]{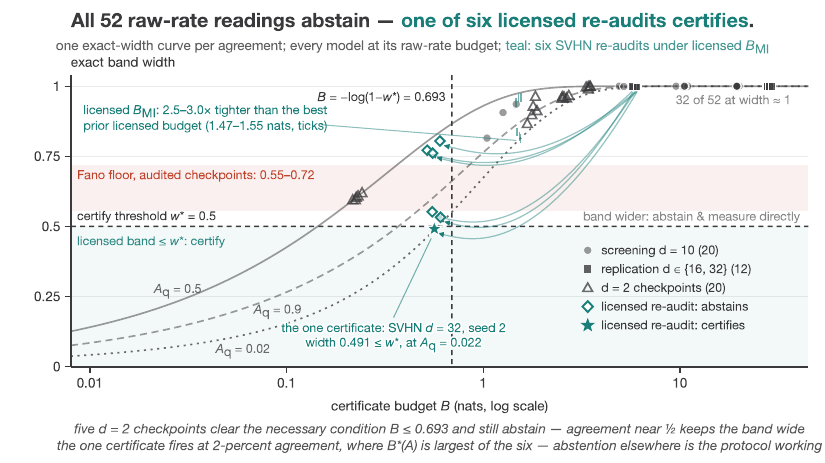}
\caption{When the certificate can fire.  Curves: exact band width, a
function of budget \emph{and} observed agreement; markers: all 52 protocol models, every one abstaining (the 20 MNIST screening
models are placed at the direct estimate of $\KL(\qbar\Vert p)$ from their
audit; at their $\Rbar$ budgets their widths are $\ge0.999$);
certifying needs a point below $w^*$.  The vertical line is
Lemma~\ref{lem:width}'s \emph{necessary} condition $B\le-\log(1-w^*)$---not
sufficient: five $d=2$ checkpoints clear it and still abstain, agreement
near $\tfrac12$ where the band is widest.  Shaded strip: Fano floor
$1-e^{-\gamma}$ on the audited checkpoints (Theorem~\ref{thm:price}d).
Teal: the six SVHN models re-audited under the licensed MI budget of
Section~\ref{sec:budgets}---$2.5$--$3.0\times$ tighter than the best
previously available licensed budget (ticks; $\Rbar$ minus the
Kolchinsky--Tracey bound on $I(x;z)$,
\citealp{kolchinsky2017estimating})---of which exactly one
(star) certifies; the six widths are those of the certified bands over the
$\hat{\mathcal A}_q$ interval of Lemma~\ref{lem:aq}.}
\label{fig:bands}
\end{figure}

\paragraph{Pipelines.}
Case study and replication: convolutional VAEs ($1{\times}28{\times}28$
for MNIST/FashionMNIST; a $3{\times}32{\times}32$ variant for SVHN);
$\beta=1$; declared likelihoods: Bernoulli-family for near-binary images
(continuous-Bernoulli for FashionMNIST, \citealp{loaiza2019continuous})
and Gaussian with fixed $\sigma=1$ for SVHN (declared before the runs, not
tuned); $K=10$; encoder map = diagonal GMM on encoder means of a
stratified 15{,}000-image train subset (per seed); decoder map =
Bregman nearest prototype in decoder-output space (Type-1 Bernoulli KL,
or squared Euclidean for the Gaussian case); natural indices; full test
splits; $S=16$; $20{,}000$ prior draws.  PPCA suite: closed-form
posteriors, 5 disjoint folds, three map constructions
(natural-Euclidean class means, regularized Mahalanobis, unsupervised
$k$-means), mean-field families in the canonical and DCT bases plus the
full-covariance control, $S=64$.

\paragraph{Per-seed replication data.}
Table~\ref{tab:perseed} gives every model of the replication grid.  Two
regime facts belong with it.  First, the declared SVHN likelihood
(Gaussian, $\sigma$ fixed to $1$) is a weak reconstruction term on these
data: the resulting rates are $\Rbar\approx5.7$--$6.1$ nats against
$27$--$29$ nats on FashionMNIST, i.e.\ the SVHN cells sit in an effectively
high-$\beta$ regime with near-collapsed sampled channels ($1$--$3$ occupied
sampled codes).  We report this as a consequence of a preregistered choice,
not as a tuned outcome; a continuous-Bernoulli or learned-$\sigma$ variant
would be a \emph{new} protocol under a new precommit, not a repair of this
one.  Second, sampled-code label accuracy on the $d{=}32$ cells is at or
near chance ($\alpha\approx0.10$ for ten classes), so the Fano floor
$\gamma(\alpha)$ is vacuous there---the codebooks in those cells carry
almost no label information, which is itself the reason the task-probe
clause cannot be met.

\begin{table}[t]
\centering
\small
\begin{tabular}{llll}
\toprule
cell & $\Aq-\Ap$ per seed & task probe post$\to$prior & verdict \\
\midrule
Fashion $d{=}16$ & $+0.109$ / $+0.034$ / $+0.030$ & falls $2/3$, rises $1/3$ & null \\
SVHN $d{=}16$    & $+0.033$ / $-0.001$ / $+0.068$ & rises $3/3$ & null \\
Fashion $d{=}32$ & $-0.035$ / $-0.033$ / $-0.065$ & rises $3/3$ & null (reversed) \\
SVHN $d{=}32$    & $-0.010$ / $+0.027$ / $-0.013$ & rises $3/3$ & null \\
\bottomrule
\end{tabular}
\caption{Preregistered replication of the case-study phenomenon: all four
cells fail the pre-declared criterion (gap $>0.05$ of one sign \emph{and}
probe drop, each in $3/3$ seeds).  SVHN probe columns use the precommitted
10-epoch evaluator (disclosed deviation, this appendix);
verdicts are identical under either evaluator.  Per-seed values:
Table~\ref{tab:perseed}.}
\label{tab:null}
\end{table}

\begin{table}[t]
\centering
\small
\begin{tabular}{llcccccccc}
\toprule
data & $d$ & seed & $\Rbar$ & $\Aq$ & $\Ap$ & $\Amu$ & occ. & $\alpha$ & $\gamma(\alpha)$ \\
\midrule
Fashion & 16 & 0 & 27.84 & 0.251 & 0.142 & 0.631 & 10 & 0.229 & 0.071 \\
Fashion & 16 & 1 & 28.00 & 0.158 & 0.124 & 0.615 & 2 & 0.136 & 0.006 \\
Fashion & 16 & 2 & 27.46 & 0.158 & 0.128 & 0.680 & 5 & 0.171 & 0.024 \\
Fashion & 32 & 0 & 27.70 & 0.098 & 0.132 & 0.467 & 1 & 0.100 & $\le0$ \\
Fashion & 32 & 1 & 28.56 & 0.126 & 0.158 & 0.524 & 2 & 0.105 & 0.000 \\
Fashion & 32 & 2 & 29.41 & 0.170 & 0.235 & 0.551 & 1 & 0.100 & $\le0$ \\
SVHN & 16 & 0 & 6.07 & 0.334 & 0.301 & 0.142 & 2 & 0.179 & $\le0$ \\
SVHN & 16 & 1 & 5.91 & 0.062 & 0.063 & 0.168 & 3 & 0.069 & $\le0$ \\
SVHN & 16 & 2 & 5.98 & 0.333 & 0.265 & 0.098 & 1 & 0.067 & $\le0$ \\
SVHN & 32 & 0 & 5.77 & 0.032 & 0.043 & 0.160 & 2 & 0.194 & $\le0$ \\
SVHN & 32 & 1 & 5.72 & 0.282 & 0.255 & 0.211 & 1 & 0.067 & $\le0$ \\
SVHN & 32 & 2 & 6.05 & 0.022 & 0.035 & 0.163 & 1 & 0.078 & $\le0$ \\
\bottomrule
\end{tabular}
\caption{\textbf{In $9/12$ models the sampled codes carry no usable label
information at all, and the case study's precondition---an informative,
non-degenerate coupled channel---was never reproduced here}; the $\gamma$
column is the sharpest summary of the regime.  Every model of the
preregistered replication grid (no selection).  $\Amu$ is the mean reading;
``occ.''\ counts occupied sampled codes above $1\%$ mass; $\alpha$ is
held-out accuracy of the majority-label predictor on \emph{sampled} codes
and $\gamma(\alpha)$ the Fano floor it implies.  Task-probe values, which
depend on the external evaluator, are in the ledger for both evaluators.}
\label{tab:perseed}
\end{table}

\paragraph{The padded arm (outside the $O(d)$ statement).}
Padding the latent space from $d=10$ to $32$ before rotating is not an
$O(d)$ action, and Theorem~\ref{thm:price}(b) does not cover it.
Empirically the pad-only arm behaves like the canonical one (budget
$0.000$, band width $0.000$, $\mathcal A[\bar p_\theta]=0.9075$), while
pad-then-rotate drives the mean-field budget to $29.4$ nats and the band to full
width, with $\mathcal A[\bar p_\theta]=0.9074$ still unmoved; the
full-covariance control certifies in every one of these arms (budget $0.000$).
Prior-draw agreement was not measured on the padded arms, so this arm speaks to
Theorem~\ref{thm:certification}(a) only.  We report the padded arm as a labelled
stress case, not as evidence for the theorem.

\paragraph{Preregistration and its three disclosed deviations.}
The replication grid and the protocol run were specified in a precommit
document---model sets, readings, verdict criteria---before execution
(reproduced verbatim in Figure~\ref{fig:precommit}).  Three departures
are disclosed here: two found afterwards by an adversarial audit of this
manuscript, and one introduced by that audit's own correction.

\begin{figure}[t]
\begin{tcolorbox}[precommitbox, title={\texttt{PRECOMMIT\_W2.md}:
precommitted readings for the replication grid (verbatim)}]
\begin{verbatim}
- R-repro: the posterior-prior contrast reproduces in a cell if A_q - A_prior > 0.05
  with the same sign in 3/3 seeds AND the task probe drops under the prior arm with
  the same sign in 3/3 seeds. Report per-cell verdicts; the flagship claim survives if
  >= 3 of 4 (dataset x d) cells reproduce.
- R-null: cells that do not reproduce are reported as nulls, verbatim, and scoped into
  the paper (the finding is then dataset/dimension-conditional).
- Mean-vs-posterior arm is expected null (as at MNIST d=10); reported either way.
- ALL cells and seeds reported; numbers only from per_model CSVs.
\end{verbatim}
\end{tcolorbox}
\begin{tcolorbox}[precommitbox, title={\texttt{PRECOMMIT\_LICENSED\_MI.md}:
fixed inputs and readings for the licensed MI budget (verbatim)}]
\begin{verbatim}
## Fixed before running (no tuning afterwards)
- M = 20000 samples, c = 5 nats, delta = 0.05, w* = 0.5.
- Models: all 12 E1-ext checkpoints; SVHN reported first (the open family).
- Sampling seed 91_000 + 100*d + seed.
- A_q is read from e1ext_per_model.csv; the admission threshold is the exact
  B*(A_q), **not** the 0.693 limit.
## Readings, fixed before running
- R-certify: a model certifies iff B_MI <= B*(A_q). Report every model.
- R-progress: if no model certifies but B_MI <= 0.8 on some model, the
  estimator route is working and the next step is a tighter concentration
  (more samples) or a better clip, not a new method.
- R-stop: if B_MI > 1.4 everywhere -- i.e. no better than the
  Kolchinsky-Tracey bound already measured -- the MI-lower-bound route is
  abandoned and pillar A stops. The paper then reports the impossibility
  framing (plan risk 1).
- Whatever happens, the split-half plug-in estimate of the truth is reported
  alongside, labelled as an estimate and never as a budget.
- w* is not re-tuned under any outcome (lessons #16).
\end{verbatim}
\end{tcolorbox}
\caption{The precommitted readings, reproduced verbatim from the two
precommit documents in the supplement (non-ASCII glyphs transliterated: $\ge$, $\delta$, $\times$, $\cdot$, dashes;
Markdown backticks, blank lines and the bold markers around the reading
names removed).  The replication readings
(top) fixed the verdict criteria of the replication in Appendix~\ref{app:record}
before the
runs; the licensed-budget readings (bottom) fixed $M$, $c$, $\delta$,
$w^*$, the seeds, and the certify/progress/stop readings before the run of
Section~\ref{sec:exp-protocol}.  The clip-tail correction and the
$\hat{\mathcal A}_q$ interval of Appendix~\ref{app:mi-lemmas} changed none
of these inputs; their dated addendum sits in the same document
(deviation D3).}
\label{fig:precommit}
\end{figure}

\emph{(D1) Evaluator epochs.}  The SVHN external evaluator was trained for
12 epochs where the precommit specified 10.  We retrained it at exactly 10
epochs and recomputed every classifier-dependent reading on the saved
checkpoints.  All four cell verdicts are unchanged under either evaluator:
the agreement-gap clause is classifier-independent and fails on its own in
all four cells, and the probe clause also fails at 10 epochs, with the
probe \emph{rising} post$\to$prior in $6/6$ SVHN models.
Table~\ref{tab:null} quotes the 10-epoch (precommitted) readings.
Absolute probe values are evaluator-sensitive and are never quoted here as
evaluator-free quantities.

\emph{(D2) Protocol population.}  The precommit fixed the
Certify-or-Measure inputs as the two screening families (jointly, the 20
MNIST screening checkpoints of Appendix~\ref{app:record}) plus the
replication models.  The first run omitted the replication models; we have
since run the protocol on all twelve of them (reported in Appendix~\ref{app:loosest}: abstain $12/12$, fresh-draw consistency
$12/12$).  Separately, the 20 two-dimensional checkpoints were \emph{not}
in the precommit and are labelled in the text as a post-hoc extension;
under licensed budgets they all abstain, and they serve to probe the
certify boundary and to demonstrate Lemma~\ref{lem:budgets}'s unlicensed
failure modes.  Under the budgets fixed in the original precommit, the
preregistered population abstains everywhere; the one certificate that
fires (Section~\ref{sec:exp-protocol}) does so on a preregistered model
but under the licensed MI budget, whose estimator, clip, seeds, and
readings were fixed in a separate, later precommit before that run.  We
consider this the honest reading and have not restated either as an
original preregistered result.

\emph{(D3) The agreement interval.}  The precommit's certify reading
compared $B_{\mathrm{MI}}$ with $B^*$ at the point estimate
$\hat{\mathcal A}_q$ read from the evaluation run.  The interval treatment
of Lemma~\ref{lem:aq}---$\delta_q=0.05$, Hoeffding over the $NS$ posterior
draws as the primary interval, the reading under which the test split is
the reference distribution---was added after the run, on 2026-09-08, in
the same dated addendum as the clip correction; the per-point variants
were computed at the same time and are reported above, including the one
under which the model abstains.

A further minor issue (one verification script overwrote its output file
instead of appending) was fixed without data loss.  All are recorded in
the results ledger shipped with the supplement.

\paragraph{The mean-vs-sampled divergence and its control, in full.}\label{app:control}
The replication runs show the \emph{mean} and \emph{sampled} readings of
the same encoder map diverging on FashionMNIST in $6/6$ models
($\Amu-\Aq\in[+0.37,+0.52]$), where at MNIST $d=10$ they agreed.  Under
the pre-declared control (refit the same family on \emph{sampled} latents,
change nothing else) the divergence collapses to $[+0.002,+0.012]$ on
FashionMNIST---$0/6$ survive---and to $[+0.007,+0.119]$ on SVHN: a
map-fitting artifact.  The control also removes the sampled-channel
collapse (all 12 models then use $10/10$ codes), so the degeneracy
confounding the null belongs to the \emph{declared map construction}, not
the trained models---that does not revive the null, but it names what a
revised protocol must change.
Under the mean-fitted map, mean-side code usage never collapses ($10/10$
occupied codes in all 12 models, usage entropy $2.20$--$2.28$ nats) while
sampled-side usage collapses in $11/12$ models to at most five codes ($10/12$
to at most three).  The control refits the same mixture family, with the same
$K$, initialization seed and regularizer, on one sampled latent per training
example (seed series $71{,}000+100d+s$), leaves the checkpoint, the decoder
\emph{rule}, the test split and the test draws untouched---the decoder
prototypes are re-derived from the refitted mixture's means, exactly as the main
run derives them, so one analyst choice moves both readers---and
recomputes both readings.  Per-model results are in the supplement; the
summary is that the divergence collapses to $\le0.012$ on all six
FashionMNIST models and to $\le0.119$ on SVHN, and that every one of the 12
models then uses all $10$ codes on sampled latents.  We had pre-declared that
survival in $\ge5/6$ FashionMNIST models would keep the finding; $0/6$
survive, so the finding is downgraded to a statement about map transfer.

\paragraph{Two live demonstrations that budget validity is load-bearing.}
A grid plug-in of $\KL(\qbar\Vert p)$ certifies all 20 $d=2$ checkpoints with
band widths $0.061$--$0.493$, and every band happens to contain the realized
$\Ap$; under the licensed $\Rbar$ the same checkpoints certify $0/20$.  A
binned KL is a data-processing lower bound, so that budget was never licensed
and the agreement of its bands with the truth was luck.  A moment-matched
Gaussian plug-in on the PPCA panel gives bands of width $0.111$--$0.181$ that
\emph{exclude} the realized $\Ap$ in $30/30$ conditions; there the error was
visible in the data.  Both are instances of Lemma~\ref{lem:budgets}.

\paragraph{The admissible agreement window, and where the shortfall lives.}\label{app:window}
Because $B^*$ is strictly decreasing on $[0,1/2]$, a budget caps the agreement it
can ever admit, and the cap is computable.  At the certifying model's
$B_{\mathrm{MI}}=0.565$ the protocol can admit only $\Aq\le0.0281$ or,
symmetrically, $\Aq\ge0.9719$; at the best
of the six, $0.519$, only $\Aq\le0.0420$.  Certifying the case study's own
$\Aq=0.770$ would need $B\le0.2073$, and $\Aq=0.5$ needs $B\le0.1438$.  The
residual between $B_{\mathrm{MI}}$ and an unbiased point estimate is
$0.058$--$0.061$ nats, so even a perfect estimator of $I(x;z)$---budget
$0.565-0.058=0.507$---would admit only $\Aq\le0.0459$.  The remaining gap is
therefore not estimator-side: rate-budget certification of generative codebook
consistency at useful agreement needs aggregate-posterior mismatch below roughly
$0.14$--$0.21$ nats, which is a target for the prior-flexibility line
\citep{tomczak2018vampprior,bauer2019resampled} rather than for a sharper bound on
$I(x;z)$.  The gap side is not subject to this cap: since
$\min_aB^*(a)=B^*(0.5)=0.1438$, the trained MLP encoder of
Section~\ref{sec:exp-price} at $\bar\Delta=0.088$ nats certifies at every
agreement, with band $[0.7335,0.9815]$ (width $0.2480$) at its observed
$\Aq=0.9016$ and width at most $0.4017$ over all agreements.

\paragraph{The licensed MI budget.}
$M=20{,}000$ samples, statistic clipped below at $c=5$ nats,
$\delta_B=0.05$, sampling seeds fixed before the run together with the
readings.  The lower clip raises the mean of the statistic by at most
$e^{-5}=0.0067$ nats (Lemma~\ref{lem:clip}), which is subtracted from the
empirical-Bernstein lower bound; in the samples the clip binds on at most
$3$ of the $20{,}000$ draws per model and the empirical tail mass is below
$10^{-4}$ nats, so the analytic correction is conservative.  The clip and
the empirical-Bernstein constant are the only free choices; the
concentration term plus the correction is $0.058$--$0.061$ nats on the
SVHN models, i.e.\ the whole residual between $B_{\mathrm{MI}}$ and an
unbiased point estimate of the truth.  The precommit stated that the clip \emph{lowers} the
mean; the sign error was found by an adversarial audit after the run, is
recorded as a dated addendum to the precommit, and its correction raised
every budget by $0.0067$ nats without changing any verdict.  The precommit
also promised the split-half plug-in estimate of $\KL(\qbar\Vert p)$
alongside: it is $0.244$--$0.428$ nats on the six SVHN models
(\texttt{kl\_plugin\_split.csv}), below the point estimates
$\Rbar-\bar T_c=0.459$--$0.548$ nats to which the residual above refers;
it evaluates the log-density of a random half of the mixture, which by
Jensen underestimates $\log\qbar$ on average, so it is biased downward
and is quoted as a lower reference, never as a budget.

\emph{The agreement estimate.}  $\hat{\mathcal A}_q$ averages $S=16$
posterior draws at each of the $N=26{,}032$ test points, $NS=416{,}512$
independent indicators; the Hoeffding interval of Lemma~\ref{lem:aq} is
$\pm0.0021$ at $\delta_q=0.05$, and the certified band is the union over
it, width $0.491$ against $0.487$ at the point estimate.  Sensitivity: if
the test points are instead regarded as an i.i.d.\ sample, so that the
draws sharing a point are dependent through it, the empirical-Bernstein
interval on the $N$ per-point means \citep{audibert2009exploration} is
$\pm0.0022$ and the band still certifies (width $0.491$); Hoeffding on
the per-point means ($\pm0.0084$) gives width $0.503$ and would abstain.
These variants change only the agreement interval; the budget remains the
test-split quantity, so under the i.i.d.\ reading they gauge the
verdict's sensitivity to the interval rather than constitute certificates.
At $\delta_B=\delta_q=0.025$ (coverage $0.95$) the width is $0.494$;
under a Bonferroni correction over the six SVHN models (family-wise
$0.90$) it is $0.499$---a certificate with no margin to spare, which is
the honest reading of a two-percent-agreement channel.  All variants are
tabulated in the supplement (\texttt{aq\_interval.csv}).

The comparison baseline---the best previously available licensed budget,
$1.467$--$1.555$ nats---is $\Rbar$ minus the Kolchinsky--Tracey
Bhattacharyya lower bound on $I(x;z)$, computed in closed form for the
Gaussian-mixture $\qbar$ \citep{kolchinsky2017estimating} and licensed by
the same rate decomposition (Lemma~\ref{lem:budgets}); the closed-form
$\log(1+\chi^2(\qbar\Vert p))$ budget was also computed and is looser on
every SVHN model ($8.1$--$15.6$ nats).

\paragraph{Attribution of measured quantities.}
Sampled-code label accuracies use held-out true labels with train-side
majority votes only (the Fano chain of Theorem~\ref{thm:price}(d) requires
the label, not a code-derived proxy); plug-in KL budgets are used only in
the conservative direction established by Lemma~\ref{lem:budgets};
and per-map, per-arm ranges are reported without selection---including the
reversed-direction cells and the mixed-sign SVHN divergence that
accompany the FashionMNIST findings.

\paragraph{Relation to the prior preprint, and attribution of
Theorem~\ref{thm:certification}.}
The coupled channel, the post-processing decomposition and the gap-side
identity and bound of Theorem~\ref{thm:certification}(a), and a
marginal-impossibility result, appear in a prior preprint
\citep{hayashi2026codebook} as its Lemma~7 and Corollary~8 and its
Proposition~4.  That Proposition matches the two code histograms, their
entropies, the occupied-code counts and $I(x;z)$; its Remark~5 records that a
Hungarian-matched score can hide the deranged case.  Relative to it, the
following are new here.  (i)~Definition~\ref{def:package} closes the hole
Remark~5 names and widens the package well past it: the matched pair here also
agrees on \emph{every} permutation-symmetric functional of the joint code
table---including the permutation-matched agreement itself, so no
relabeling-invariant summary separates the pair---and on the \emph{per-example}
reconstruction, rate, ELBO, $\log p(x)$ and variational gap, on the generative
law $p(x)$ as a distribution, on the per-dimension active-unit statistics, on
$\Ap$, and (item~(I$'$)) on all latent moments below a designable order.
(ii)~Theorem~\ref{thm:certification}(b), the rate-side bound, and with it the
two-sidedness, is new.
(iii)~Theorem~\ref{thm:price} in all of its parts, Lemma~\ref{lem:width},
Lemma~\ref{lem:budgets} and the licensed mutual-information budget, the admission budget $B^*(a)$ and the Certify-or-Measure protocol have no
counterpart in the preprint, and neither does any audit reported here.  Two
model populations are shared with it: the five MNIST models of the case study
(Section~\ref{sec:exp-e1}), which are also the $\beta=1$ arm of the 20
screening models, and the 20 two-dimensional checkpoints of
Section~\ref{sec:exp-protocol} are trained models audited in the preprint.
For the two-dimensional checkpoints, the encoder-side agreement and the rate
are taken from the preprint's audit tables, the agreement re-verified here by
recomputation to a relative error below $10^{-7}$; every other quantity
reported for these models is computed here.  The derivations
here are self-contained.

\end{document}